\documentclass{article} % For LaTeX2e
\usepackage[final]{colm2025_conference}

\usepackage{microtype}
\usepackage{hyperref}
\usepackage{url}
\usepackage{booktabs}
\usepackage[toc,page]{appendix}
\usepackage{titletoc}
\usepackage{lineno}

\definecolor{darkblue}{rgb}{0, 0, 0.5}
\hypersetup{colorlinks=true, citecolor=darkblue, linkcolor=darkblue, urlcolor=darkblue}

\usepackage{subcaption}
\usepackage{microtype}
\usepackage{graphicx}
\usepackage{booktabs} % for professional tables

\usepackage[dvipsnames]{xcolor}
\usepackage{hyperref}
\usepackage{url}

\usepackage{graphicx} % for including graphics
\usepackage{caption}
\usepackage{subcaption} % for subfigures

\usepackage{hyperref}
\usepackage{url}
\usepackage{makecell}
\usepackage{varwidth}
\usepackage{enumitem}
\usepackage{tabularx}
\usepackage{booktabs}
\usepackage{graphicx}
\usepackage{multirow}
\usepackage{pdflscape}
\usepackage{colortbl}%
  
\definecolor{softgreen}{rgb}{0.7, 0.9, 0.7} 
\definecolor{softorange}{rgb}{1.0, 0.8, 0.6}

\usepackage[normalem]{ulem}

\usepackage{booktabs}
\usepackage{wrapfig}

\usepackage{algorithm}
\usepackage{algpseudocode}
\usepackage{wrapfig}
\usepackage{lipsum}

\usepackage[utf8]{inputenc} % allow utf-8 input
\usepackage[T1]{fontenc}    % use 8-bit T1 fonts
\usepackage{hyperref}       % hyperlinks
\usepackage{url}            % simple URL typesetting
\usepackage{booktabs}       % professional-quality tables
\usepackage{amsfonts}       % blackboard math symbols
\usepackage{nicefrac}       % compact symbols for 1/2, etc.
\usepackage{microtype}      % microtypography
\usepackage[normalem]{ulem}
\usepackage{tikz}

\usepackage{amsmath}
\usepackage{amssymb}
\usepackage{mathtools}
\usepackage{amsthm}
\usepackage{tcolorbox}
\usepackage[capitalize,noabbrev]{cleveref}

\theoremstyle{plain}
\newtheorem{theorem}{Theorem}[section]

\theoremstyle{definition}

\theoremstyle{remark}
\newtheorem{remark}[theorem]{Remark}

\usepackage[textsize=tiny]{todonotes}

\newcommand{\autoscale}{\textsf{AutoScale}}

\usepackage{titletoc}

\title{\autoscale: Scale-Aware Data Mixing for Pre-Training LLMs}

\author{%
  Feiyang Kang\thanks{Equal contribution. $^\dagger$Correspondence to: Feiyang Kang and Ruoxi Jia \texttt{$<$fyk, ruoxijia$>$@vt.edu}.  } $\,^{\dagger}$ \\
  Virginia Tech\\
  \small{\texttt{fyk@vt.edu}}\normalsize \\
  \And
  Yifan Sun$^{*}$\\
  UIUC\\
  \small{\texttt{yifan50@illinois.edu}}\normalsize \\
  \And
  Bingbing Wen\\
  University of Washington\\
  \small{\texttt{bingbw@uw.edu}}\normalsize \\
  \And
  Si Chen\\
  Virginia Tech\\
  \small{\texttt{chensi@vt.edu}}\normalsize \\
  \AND
  Dawn Song\\
  UC Berkeley\\
  \small{\texttt{dawnsong@gmail.com}}\normalsize \\
  \And
  Rafid Mahmood\\
  University of Ottawa \& NVIDIA\\
\small{\texttt{mahmood@telfer.uottawa.ca}}\normalsize
  \And
  Ruoxi Jia$^{\dagger}$\\
  Virginia Tech\\
  \small{\texttt{ruoxijia@vt.edu}}\normalsize \\
}

\begin{document}

\ifcolmsubmission
\linenumbers
\fi

\maketitle

\begin{abstract}
Domain reweighting is an emerging research area aimed at adjusting the relative weights of different data sources to improve the effectiveness and efficiency of LLM pre-training. We show that data mixtures that perform well at smaller scales may not retain their advantage at larger scales, challenging the existing practice of determining competitive mixtures in small-scale experiments and \emph{directly} applying them at much larger scales. To address this, we propose \autoscale, a two-stage, scale-aware data composition framework. First, \autoscale~fits a parametric model that predicts the model’s loss under different data compositions, then uses it to find an approximate best allocation at smaller, more manageable budgets. Next, leveraging a novel theoretical analysis of how optimal compositions evolve with scale, \autoscale~extrapolates that composition to larger budgets without further retraining. Empirically, \autoscale~accelerates convergence and improves downstream performance.
For instance, when pre-training GPT-2 Large, it achieves a 28\% faster perplexity reduction than baselines and up to a 38\% speed-up over unweighted training, while yielding best-average results on various downstream tasks. Overall, our findings illustrate how domain importance shifts with training scale, underscoring the need for scale-dependent data curation in LLM training. 
Our code is open-sourced\footnote{\small{\url{https://github.com/feiyang-k/AutoScale}}\normalsize}.
\end{abstract}

\section{Introduction}
Large language models (LLMs) are pre-trained on vast datasets sourced from diverse domains. However, the immense computational demands of this process, coupled with limited resources, create a pressing need to enhance the effectiveness and efficiency of pre-training. A promising approach to address this challenge is through \emph{domain reweighting}---adjusting the ratio (or weights) of data from different sources.

However, developing a principled and efficient framework for determining an optimal data mix remains challenging. Many industry pipelines still rely on trial-and-error heuristics~\citep{rae2021scaling,grattafiori2024llama} or reuse domain weights designed for previous models~\citep{mehta2024openelm}, without a systematic approach for
deciding how much of each domain to include. The seminal domain-optimization work by~\citet{xie2024doremi} attempted to upweight ``difficult'' domains, but later work~\citep{fan2023doge} reported instability and only limited validation-loss improvements, partly because the chosen optimization objective does not robustly align with the model’s ultimate test-time performance. Therefore, recent methods~\citep{liu2024regmix,ye2024data} focus on \emph{directly} optimizing domain weights for lower validation loss. However, the highly complex relationship between domain weights and model performance makes such optimization expensive. A common strategy of these works is to reduce costs is to train multiple times at smaller scales, identify a ``best'' mix, and then assume it transfers to large-scale pre-training . Yet our experiments show that compositions found at smaller scales may not remain competitive when training is scaled up, whereas directly optimizing at full scale is infeasible. This yields a \emph{dilemma}: either accept small-scale solutions that may not transfer or attempt large-scale optimization that is prohibitively costly.

To resolve this dilemma, we develop a novel theoretical analysis that shows how the optimal domain composition evolves at different scales. Building on this insight, we propose a two-stage framework, \autoscale, for scale-aware domain reweighting. In the first stage, we fit a parametric model that predicts the model’s loss under different data compositions, then use it to discover an approximate optimal allocation at smaller, more manageable budgets. Next, we apply our theoretical result to extrapolate that allocation to larger budgets---without re-optimizing at full scale---thus bridging the gap between small-scale optimization and full-scale pre-training.

Our experiments, conducted on both decoder-only and encoder-only architectures, consistently that \autoscale~speeds up convergence and yields favorable downstream task performance. For instance, in pre-training GPT-2 Large on the RedPajama dataset, our approach achieves a 28\% faster perplexity reduction compared with any baseline and up to a 38\% speed-up over unweighted training, while also delivering the best downstream-task performance. Moreover, we made surprising empirical observations that data sources  traditionally viewed as ``high-quality'' (e.g., Wikipedia and scientific papers) excel at smaller scales but exhibit sharp diminishing returns as the training grows. Meanwhile, domains containing more diverse examples (e.g., CommonCrawl) continue offering loss reductions at larger scales, underscoring the importance of scale-aware data curation.

\vspace{-0.7em}\section{Related Work}\vspace{-0.5em}

 Principled training data curation for LLMs is an emerging research area, aiming to strategically select data that improves model performance. It can be performed at multiple levels---from token-level~\citep{lin2024rho} or point-level~\citep{wang2024greats} up to domain-level selection. Domain-level approaches are often more efficient because they operate at a coarser granularity, typically applying soft selection---i.e., upweighting or downweighting entire data domains.

Domain reweighting can generally be viewed as a two-step process: (i) define an objective that captures the goal of improving test loss or other model performance measures, and (ii) optimize domain weights according to that objective. DoReMi~\citep{xie2024doremi}, a seminal paper in the space, adopted GroupDRO~\citep{sagawa2019distributionally} as the objective, which implicitly upweights ``difficult'' domains. However, subsequent studies~\citep{fan2023doge} found that the performance gains are unstable and limited, partly because the chosen objective does not align with the metrics ultimately used to evaluate the model at test time. Recent methods~\citep{liu2024regmix,ye2024data,fan2023doge} attempt to directly optimize validation loss, which serves as a closer proxy for the test-time model performance metrics we actually care about. We note that there can still be misalignments between validation loss and test-time performance metrics~\citep{mosaic_evaluation_gauntlet}, which is an active research area in itself, but validation loss remains a widely accepted objective for model selection in pre-training.

With validation loss as the objective, the next challenge is how to optimize it. Evaluating the objective for a given set of weights is computationally expensive, as it requires training a model from scratch on the weighted data set and evaluating the corresponding validation loss. Existing approaches tackle this in two broad ways. One line uses surrogate models to approximate the mapping from domain weights to model performance~\citep{ye2024data,liu2024regmix}; fitting such models can still require large amounts of retrainings. For instance, in \citet{liu2024regmix}, fitting the surrogate requires more than ten times as many training runs as there are domains. Another line performs local approximations, assuming only a single gradient step for the underlying model~\citep{fan2023doge}, which may not hold for practical learning rates.

Our work follows the surrogate-modeling line but differs by proposing a new parametric function to model performance versus domain weights, which can be reliably fit with only about twice as many retraining runs as the number of domains. Crucially, existing methods often conduct this optimization at a smaller data scale to keep cost manageable, then directly apply the small-scale ``best'' mixture at large scale. We show that this scale-invariance assumption can break down in practice. To address this, we contribute a novel theoretical analysis that characterizes how the optimal domain ratio shifts across scales, enabling us to extend small-scale insights effectively to larger budgets.

Finally, the general notion that data curation should be scale-dependent has appeared in prior works~\citep{sorscher2022beyond,goyal2024science}. However, \citet{sorscher2022beyond} argues this point using a simplified analysis with a perception model---showing that larger scales favor ``harder'' samples while smaller scales favor ``easier'' samples---but does \emph{not} propose a practical pipeline for scale-aware data selection. Our work addresses this gap by introducing a concrete method for scale-dependent curation at foundation-model scale. Meanwhile, \citet{goyal2024science} focuses on the CLIP model~\citep{radford2021learning}, finding that different training epochs call for different data-selection thresholds. By contrast, our setting involves (i) LLM pre-training (with typically one epoch), and (ii) different data modality and selection granularity.

\section{Methodology} 
% LLMs are often trained on data from diverse domains---such as Wikipedia, scientific papers, and large web crawls. These sources typically receive fixed proportions set by heuristics or small‐scale tests, yet it remains unclear if the same ratios remain competitive when training data scales up.

\begin{wrapfigure}{R}{0.4\textwidth}\vspace{-4.5em}
\begin{minipage}{0.4\textwidth}
    \scalebox{0.39}{\includegraphics{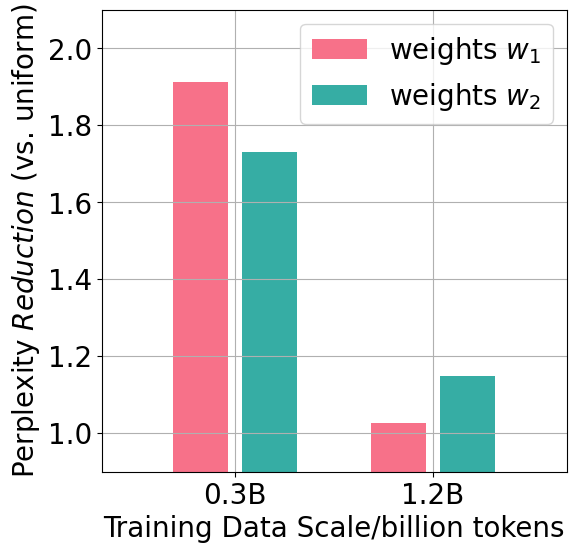}}\vspace{-0.2em}
     \caption{\small{
     % The optimality of domain weights is \textit{scale-dependent}. 
     \textbf{Domain weights that excel at one scale may underperform at another}. Weights $w_1$ and $w_2$ are obtained by running DDO (as introduced in Section~\ref{sec:ddo_text}) at 0.3B and 1.2B, respectively.
     % Domain reweighting methods not considering this scale-dependency may show instable performance in applications. (The scaling law representation DDO employs accurately models the evaluation loss of the domain weights–with an \textit{\textcolor{teal}{Average Absolute Relative Error (AAR) $=1.00\%$}} between the predicted evaluation loss and measured value on the trained models.
     }}\normalsize
    \label{fig:evidence}\vspace{-3em}
\end{minipage}
\end{wrapfigure}

\paragraph{Evidence of scale dependence.}  A simple experiment illustrates that domain weighting is not one‐size‐fits‐all: we derive two data mixes (with a procedure introduced later) and compare them at different training budgets. As shown in Figure~\ref{fig:evidence}, when tested at 0.3B tokens, Mix A beats Mix B as measured by validation perplexity reduction compared against uniform weights, but at 1.2 B tokens Mix B outperform Mix A. This flip indicates that a domain which helps more at a smaller scale may not remain a better choice at a larger scale, while another domain initially less impactful can become more valuable as training grows. Consequently, a scale‐aware approach is needed so that domain weights adapt as more tokens are introduced.

\vspace{-0.5em}\subsection{Problem Formulation}

To capture how domain importance shifts with the total training budget, we now formalize the scale‐dependent domain reweighting problem. This framework lets us solve for a better domain mixture at any given data scale.

\vspace{-0.5em}\paragraph{Notations and setup.}
We consider $m$ domains $\{D_1, \dots, D_m\}$, each with a large pool of training examples. A \emph{domain mix} is specified by a weight vector $\mathbf{w} = [w_1, \dots, w_m]^\top$ on the probability simplex $\mathbb{W}^m 
  \;:=\; 
  \bigl\{ 
    \mathbf{w}\in \mathbb{R}^m 
    \,\big\vert\, 
    \sum_{i=1}^m w_i = 1,\,
    w_i \ge 0 
    \text{ for all } i 
  \bigr\}$.
Given a total budget of $N$ tokens, let $N_i = \lfloor w_i \cdot N \rfloor$ be the number of tokens chosen from domain $D_i$. We denote this resulting dataset by
$
  S(N, \mathbf{w}) \;=\; \{ S_1, \dots, S_m \},
$
where $S_i\subseteq D_i, \lvert S_i \rvert = N_i$. Training a model $\boldsymbol{\theta}$ on $S(N, \mathbf{w})$ means solving an empirical risk minimization (ERM) objective:
$
  \boldsymbol{\theta}^*(N, \mathbf{w})
  \;=\;
  \arg\min_{\boldsymbol{\theta}}
  \;\mathcal{L}\bigl(\boldsymbol{\theta},\,S(N,\mathbf{w})\bigr),
$
where $\mathcal{L}$ is a next-token prediction loss.

\vspace{-0.0em}\paragraph{Objective function.} We assess a domain mix $\mathbf{w}$ by measuring the validation loss
$
  \mathcal{L}^v\bigl(\boldsymbol{\theta}^*(N,\mathbf{w})\bigr)
  \;=\;
  \mathcal{L}\bigl(\boldsymbol{\theta}^*(N,\mathbf{w}),\,D^v\bigr)
$
on a held-out dataset $D^v$. 
We then seek the mix $\mathbf{w}$ that \emph{minimizes} this validation metric at scale $N$:
\begin{equation}
\label{eq:bilevel}
  \mathbf{w}^*
  \;=\;
  \arg\min_{\mathbf{w}\in \mathbb{W}^m}
  \;\;
  \mathcal{L}^v\bigl(\boldsymbol{\theta}^*(N,\mathbf{w})\bigr).  
\end{equation}
Because $\boldsymbol{\theta}^*$ depends on $\mathbf{w}$ via ERM, this becomes a \emph{bi-level optimization} problem. Since no closed-form expression exists for $\boldsymbol{\theta}^*$, any gradient-based approach 
must approximate 
$\tfrac{\partial}{\partial \mathbf{w}}\,\mathcal{L}^v\bigl(\boldsymbol{\theta}^*(N,\mathbf{w})\bigr)$. Traditional bi-level methods rely on higher-order derivatives with respect to model parameters for this approximation, but such techniques become infeasible at the scale of modern foundation models~\citep{liu2021investigating}.

\vspace{-0.5em}\subsection{Our Solution}\vspace{-0.5em}
We propose a \textbf{two-stage} framework, \autoscale, for finding scale‐aware data compositions, which first approximates the optimal data mix at small scales and extrapolates to a larger target scale:
\begin{enumerate}
    \item Direct Data Optimization (\textsc{DDO}): At smaller, computationally feasible scales, we learn a mapping from domain weights to validation loss. This reduces the original bi-level problem to a single-level convex optimization---letting us approximate the “best” domain mix for that smaller budget.
    \item Optimal Mix Projection: Building on a theoretical analysis of how domain importance changes with total tokens, we then extrapolate those small‐scale DDO solutions to a larger data budget.
\end{enumerate}

\subsubsection{Direct data optimization} 
\label{sec:ddo_text}
Direct Data Optimization (DDO) is a practical method for approximating the solution to the bi-level domain‐weighting problem at relatively small data scales. The key idea of DDO is to approximate the validation loss $\mathcal{L}^v\bigl(\boldsymbol{\theta}^*(N,\mathbf{w})\bigr)$ as a \emph{parametric function} of the domain‐weight vector $\mathbf{w}$. This effectively reduces our bi‐level objective (choose $\mathbf{w}$ while also training $\theta$) to a \emph{single‐level} optimization, which can be solved efficiently via standard gradient‐based methods.

% \kang{fix}We begin by noting that the total validation loss can be decomposed by domains:
% \[
%   \mathcal{L}^v \bigl( \boldsymbol{\theta}^*(N,\mathbf{w}) \bigr)
%   \;=\;
%   \sum_{i=1}^m
%   \mathcal{L}^v_i \bigl( \boldsymbol{\theta}^*(N,\mathbf{w}) \bigr),
% \]
% where $\mathcal{L}^v_i$ denotes the validation loss evaluated specifically on domain $i$.
% Next, we \emph{model} the dependence of each domain’s validation loss on $\mathbf{w}$, then \emph{sum} their contributions.

We begin by noting that the validation loss can be \emph{represented} by scaling laws as a function of data size for each individual domain. We \emph{model} the dependence of validation loss on the size of data from each domain, then \emph{aggregate} these functions to derive the final approximation for the validation loss on $\mathbf{w}$.

Drawing inspiration from neural scaling laws---which indicate a power-law relationship between training data scale and validation loss~\citep{kaplan2020scaling}---we assume that validation loss as a function of domain $i$'s data  size follows
\[
  \mathcal{L}^v \bigl(\boldsymbol{\theta}^*(N,\mathbf{w}) \bigr)
  \;\approx\;
  \bigl(N_0^i + w_i \cdot N \bigr)^{-\gamma_i} \;+\;\ell_i.
\]
Here, $w_i$ denotes the fraction of the total token budget $N$ allocated to domain $i$.
The term $N_0^i$ represents an the ``equivalent data size'' contributed by domains other than $i$,
while $\gamma_i$ governs how quickly domain $i$ reaches a point of diminishing returns.
Lastly, $\ell_i$ represents the irreducible term in the loss function.

To learn these parameters $\{N_0^i, \gamma_i, \ell_i\}$ for each domain $i$, we retrain the model after perturbing $w_i$ upward and downward, measure the change in total validation loss, and then fit $(N_0^i + w_i \cdot N)^{-\gamma_i} + \ell_i$ via least squares. 
% aggregate effect 
Because $\mathcal{L}^v$ aggregates the effects of data size for each domain, our final approximation for the validation loss is:
\[
  \mathcal{L}^v \bigl(\boldsymbol{\theta}^*(N,\mathbf{w}) \bigr)
  \;\approx\;
  \sum_{i=1}^m 
  \Bigl( N_0^i + w_i \cdot N \Bigr)^{-\gamma_i}
  \;+\;\ell_i.
\]
Once we have fitted the parameters, we can \emph{directly optimize} over $\mathbf{w}$ subject to $\sum_{i=1}^m w_i = 1$ to approximate the optimal domain mix under the total token budget $N$.

Because DDO only requires retraining at $(2m + 1)$ mixes 
(one baseline plus up/down perturbations for each of the $m$ domains), 
it is far cheaper than a naive zero-order method that retrains the model 
at every weight update. 
Nevertheless, DDO is best suited for moderate domain counts ($m$) 
and data scales ($N$). 
For much larger target scales, we introduce a \emph{second stage} 
that extrapolates the ``best'' DDO mix from smaller scales to significantly 
bigger budgets, all \emph{without} additional retraining.

\subsubsection{Optimal mix projection}\label{sec:thm} Our method for extrapolating domain mixes to much larger training budgets hinges on a novel theoretical result that we developed to characterize how the optimal mix ratio depends on the total data scale. Note that all inverses $(\cdot)^{-1}$, products, and exponentiations on vectors below are understood \emph{elementwise}.

\begin{tcolorbox}[title=Theorem 1: Scale-Dependent Optimal Composition,colback=white,colframe=black]
Consider the optimization problem
\[
  \min_{\mathbf{N}}
    \biggl\{
      \sum_{i=1}^m \beta_i \, N_i^{-\gamma_i}
      \;\Big\vert\;
      \sum_{i=1}^m N_i = N
    \biggr\},
\]
where $\beta_i \ge 0$ and $\gamma_i \ge 0$ for all $i$, and $\mathbf{N} = (N_1,\dots,N_m)$ denotes the domain allocations.  
Let $\mathbf{N}^*(N)$ be the \emph{optimal allocation} that minimizes the sum above for a total budget $N$.
For two distinct budgets $N^{(1)} \neq N^{(2)}$, and any larger budget $N^{(3)}$, suppose there is a constant $k > 0$ such that
\[
  \mathbf{N}\bigl(N^{(3)}\bigr)
  \;=\;
  \mathbf{N}^*\bigl(N^{(2)}\bigr)
  \;\Bigl[
    \bigl(\mathbf{N}^*\bigl(N^{(1)}\bigr)\bigr)^{-1}
    \;\mathbf{N}^*\bigl(N^{(2)}\bigr)
  \Bigr]^k,
  \quad
  \text{with }
  \sum_{i=1}^m N_i\bigl(N^{(3)}\bigr) = N^{(3)}.
\]
Then $\mathbf{N}\bigl(N^{(3)}\bigr)$ is also the optimal allocation for the budget $N^{(3)}$, \emph{i.e.},
\[
  \mathbf{N}\bigl(N^{(3)}\bigr)
  \;=\;
  \arg\min_{\mathbf{N}}
  \Bigl\{
    \sum_{i=1}^m \beta_i \, N_i^{-\gamma_i}
    \;\Big\vert\;
    \sum_{i=1}^m N_i = N^{(3)}
  \Bigr\}
  \;=\;
  \mathbf{N}^*\bigl(N^{(3)}\bigr).
\]
\end{tcolorbox}\label{thm:thm1}

\paragraph{Proof overview (high-level).} At optimality, the first-order (KKT) conditions impose that each domain’s partial derivative of the loss 
matches up to a single Lagrange multiplier. 
From this, we can derive how each domain’s optimal allocation $N_i^*(N)$ scales 
when transitioning from budget $N^{(1)}$ to another budget $N^{(2)}$.
These domain-by-domain scaling factors do not depend on the absolute size of $N$, 
only on the relative shifts between domains, 
which in turn yields an exponential-style expression for the optimal allocation at a third budget $N^{(3)}$.
Thus, once we know the optimal allocations at two budgets, 
we can directly construct the optimal allocation for any larger budget 
\emph{without} re-solving the entire optimization.

\paragraph{Interpretation of the theory.} The statement above assumes each domain $D_i$ contributes $\beta_i \, N_i^{-\gamma_i}$ \emph{independently} 
to the total validation loss. In Appendix~\ref{app:theorem1}, we generalize this to cases where domains may overlap 
by treating the evaluation as composed of multiple ``latent skills''~\citep{tiong2024toward}; 
the same exponential-style scaling behavior still emerges.

We defer the full proof to Appendix~\ref{app:lemma1} (where we employ first-order optimality/KKT conditions),
but the key insight is that domains \emph{saturate} at different rates depending on their exponents $\gamma_i$.
Specifically, a domain $D_i$ with a small $\gamma_i$ saturates more slowly and thus continues to yield
benefits at larger budgets, receiving an increasingly bigger fraction of tokens as $N$ grows.
In contrast, a large $\gamma_i$ indicates that $D_i$ quickly saturates, so it is favored primarily at smaller scales.

Concretely, this means the ``optimal mix ratio'' is \emph{not} constant across all scales.
As the total budget $N$ increases, domains with smaller $\gamma_i$ are allocated a larger share.
The theorem's exponential-style update precisely captures these changing allocations,
enabling us to predict the best mix at a higher budget 
given the solutions at two smaller budgets---\emph{without} re-solving the entire optimization problem.

\subsubsection{Overall algorithm}

Having established that the optimal domain mix varies predictably with training budget, we now summarize \autoscale~, our proposed two-stage approach to optimize data mix. 

\textbf{Stage 1} (pseudocode provided in Algorithm~\ref{alg:ddo}): Pick two feasible scales $N^{(1)}$ and $N^{(2)}$ (with $N^{(1)} < N^{(2)}$), where retraining the model is still affordable. Run DDO to obtain optimal allocations $\mathbf{N}^*(N^{(1)})$ and $\mathbf{N}^*(N^{(2)})$. 

\textbf{Stage 2} (pseudocode provided in Algorithm~\ref{alg:as}): Leveraging our theoretical result, we automatically predict the optimal domain mix for any larger scale. Specifically, starting from the optimal domain allocation $\mathbf{N}^*\bigl(N^{(2)}\bigr)$, 
we repeatedly ``scale up'' by multiplying by 
$\Bigl[(\mathbf{N}^*(N^{(1)}))^{-1} \,\mathbf{N}^*\bigl(N^{(2)}\bigr)\Bigr]^\delta$. 
Each such update yields a new allocation at a larger budget than before. 
We continue until reaching or exceeding the target budget $N^{\text{tgt}}$. 
By adjusting the resolution $\delta$, we control the granularity of these updates, 
thus reaching the target scale with any desired accuracy.

\vspace{-0.5em}\section{Evaluation}\label{sec:eval}

Our evaluation aims to address the following questions: 
\begin{itemize}
    \item \textbf{(RQ1)} \emph{Does DDO yield better domain weighting at smaller scales?} 
In our approach, DDO is designed to approximate the best data mix at a given scale. While we cannot verify its optimality, we want to see if DDO meaningfully improves domain weighting compared to baseline methods (Section~\ref{sec:ddo}). 
\item \textbf{(RQ2)} \emph{Can \autoscale---DDO at smaller scales plus our theoretical projection to larger scales---achieve training efficiency and performance benefits when direct DDO at large scale is prohibitively expensive?}(Section~\ref{sec:autoscale})
\end{itemize}

% \begin{tcolorbox}[title=\text{ Overview of Takeaways.} ,colback=white,colframe=teal]
% \begin{itemize}[leftmargin=1em]

% \item \textbf{Takeaway 1:} DDO visibly decreases the model evaluation loss and improve tasks performance. Scaling-law representation records Average Absolute Relative Error (AAR) = 1.00$\%$ in prediction of test perplexity. \textbf{(RQ1)}

% \item \textbf{Takeaway 2:} Results confirm that optimality of domain weights are scale-depedent; domain weights high-performing at one scale may become low-performing when applied at other scales. \textbf{(RQ1)}

% \item \textbf{Takeaway3:} \autoscale~ decreases validation perplexity 28\% faster than baselines, with up to 38\% speed-up over unweighted training, achieving the best performance across downstream tasks. \textbf{(RQ2)}

% \item \textbf{Takeaway 4:} As the training data scale grows, AutoScale predicted weights priorities data sources with diverse examples, such as C4 and CommonCrawl. 
% \textbf{(RQ3)}

% \item \textbf{Takeaway 5:} Data sources with standard format such as Wikipedia and scientific papers, regarded as high quality, are most favorable at smaller scales and observe sharp decline in weights as the training data scales up. \textbf{(RQ3)}

% % \item Takeaway 5: Evaluated on both Decoder-only LMs and Encoder-only LMs across different model sizes and data sizes; achieving consistent results compared against 5 representative baselines. Code is open-sourced

% \end{itemize}
% \end{tcolorbox}

\vspace{-0.5em}
\paragraph{Overview of experimental settings.} We provide an overview here and defer the full details to Appendix~\ref{app:eval}. \textbf{(I) Models and datasets.} We experiment with two architectures---GPT-2~Large (774M parameters) and BERT (110M), 
training on up to 10B tokens. While this budget is comparatively small for the latest LLM regimes, 
it already \textbf{exceeds} the data scales used in many existing domain-reweighting studies~\citep{fan2023doge,chen2024skill}, 
serving as a feasible testbed for \emph{prototype} ideas in non-commercial settings. Specifically, we pre-train GPT-2~Large on the RedPajama dataset~\citep{together2023redpajama},which spans seven domains (e.g., Common Crawl, C4~\citep{raffel2020exploring}, GitHub, Wikipedia, ArXiv, StackExchange). 
For BERT, we use data from five sources---Amazon Reviews, Arxiv, Books, Wikipedia, 
and Open WebText~\citep{Gokaslan2019OpenWeb}. \textbf{(II) Baselines.} We compare our methods against several baseline strategies. \textbf{Uniform} samples data from each domain, leading to the same token count per domain. \textbf{Llama weights} are a curated set of heuristically tuned domain weights from the LLaMA-1/2 models~\citep{touvron2023llama}. \textbf{DoReMi}~\citep{xie2024doremi} is a seminal paper in this domain-reweighting space, offering an early, principled approach to finding domain weights. \textbf{Data mixing law}~\citep{ye2024data} and \textbf{RegMix}~\citep{liu2024regmix} represent the latest state-of-the-art. \textbf{(III) Metrics.} We measure \emph{test} perplexity and also evaluate downstream performance to confirm that improvements extend to practical tasks.

% We report results for our methods (\textsc{DDO} and \autoscale~) and 6 \textbf{baselines}–\textsc{Uniform}, \textsc{LLaMA weights} (curated), \textsc{DoReMi} (LLaMA weights initialization), \textsc{Data Mixing Laws from \citep{ye2024data}},  \textsc{DoReMi} from \citet{xie2024doremi} (uniform initialization), and \textsc{RegMix} from \citet{liu2024regmix}. 

% an academic/[non-commercial?] setting. 

% Further details appear in Appendices~\ref{sec:appendix_exp_details_gpt} 
% and~\ref{sec:appendix_exp_details_bert}; runtime and GPU hours are recorded in Appendix~\ref{runtime}.

% Throughout, we measure how many training steps are saved to reach a given perplexity and also evaluate downstream performance to confirm that improvements extend to practical tasks.

% Empirical studies validate the claims and arguments made in this work and showcase the efficacy of proposed methods. We summarize the findings from empirical results in the following takeaways, responding to the research questions (RQs) of interest.

\subsection{Evaluating DDO}\label{sec:ddo}

% \subsection{RQ 1: Does DDO yield better domain weights at smaller data scales?}
% Developed via efficient approximations, DDO is designed to produce practical solutions to originally intractable problems. We cannot guarantee true global optimality of domain weights yielded by DDO, though, we conduct empirical studies to examine whether DDO meaningfully improves the optimality of domain weighting compared to baseline weights.

\paragraph{Effectiveness of DDO-optimized weights.} We perform DDO on GPT-2~Large at two different data scales (0.3B and 1.2B tokens)
to obtain DDO weights specifically optimized for each scale.
We then retrain the model under these DDO-derived weights and compare the evaluation loss
against two baselines:
(1) Uniform (no reweighting), and
(2) RegMix, the latest state-of-the-art approach.
(We omit DoReMi here, as it has been surpassed by RegMix~\citep{liu2024regmix}.) For each set of weights, the model is trained at both 0.3B and 1.2B tokens,
with results in Table~\ref{tab:ddo_table1}.
At both scales, DDO-optimized domain weights significantly outperform
the Uniform baseline, achieving a notably lower evaluation loss.
DDO-optimized weights also surpass RegMix when models are trained at the same scale as the domain-weight optimization, indicating that DDO finds more effective domain weights than RegMix. Further, RegMix does not consider adaptation for training models at different scales. Applying RegMix optimized weights on larger data scales appears less effective, evident by the \textit{widening gap} between its performance from DDO's. 
Notably, the DDO-derived weights yield the strongest gains at the scale
for which they were optimized, while showing less advantage when used at a \emph{different} scale,
highlighting the \emph{scale-dependent} nature of domain weighting.

% We perform DDO to optimize domain weights for GPT-2 Large at two different training data scales, 0.3B and 1.2B tokens, respectively. Then, we re-train the models with DDO yield weights and compare the evaluation loss against the model trained with Uniform weights and the weight obtained from RegMix which represents the previous state-of-the-art. For all these weights, we train models for 0.3B and 1.2B tokens, respectively, and evaluate their loss. Results are presented in Table \ref{tab:ddo_table1}. At both data scales, models trained on DDO optimized domain weights significantly outperform models on un-optimized, uniform weights achieving lower evaluation loss with a clear margin. Looking in more detail, DDO optimized domain weights perform best at the scale they were optimized and may be overshadowed when applied to different scales, suggesting the scale dependence for the optimality of domain weights. 

\begin{table}[h!]\centering\resizebox{0.8\linewidth}{!}{
\begin{tabular}{l|cc}
\toprule
\textbf{Weights/Actual training scale} & 0.3B training tokens           & 1.2B training tokens      \\ \midrule
Uniform Weights           & 48.04            & 28.11    \\ 
RegMix Weights (optimized at 0.3B)           & 46.56 (-1.48)           & 27.86 (-0.25)   \\\midrule
DDO Weights (optimized at 0.3B)          & \textbf{46.13 (-1.91)}            & 27.09 (-1.02)     \\ 
DDO Weights (optimized at 1.2B)           & 46.31(-1.73)           & \textbf{26.97 (-1.14)}     \\ 
\bottomrule
\end{tabular}}\vspace{-0.5em}\caption{\small{GPT-2 Large trained with DDO optimized domain weights achieve significantly reduced test perplexity compared to with non-optimized, uniform weights, also outperforming RegMix. DDO optimized weights appear most performant at the data scale they were optimized. 
% (774M Decoder-only LMs; Redpajama dataset with 7 domains.) 
}}\normalsize\label{tab:ddo_table1}\vspace{-1em}
\end{table}

In addition, we apply DDO to BERT at 0.3B tokens; the resulting model performance from the DDO-optimized weights is shown in Fig.~\ref{fig:figure2}. 
These weights reduce the model's validation loss on all training domains \emph{and} on held-out 
non-training domains, demonstrating DDO's effectiveness in improving training efficiency. 
Furthermore, when evaluated on the GLUE benchmark and the SQuAD dataset, the DDO-optimized weights 
also yield a notable improvement in downstream task performance.

\begin{figure}[h!] 
    \centering
    \begin{subfigure}[b]{0.49\textwidth}
        \includegraphics[width=\textwidth]{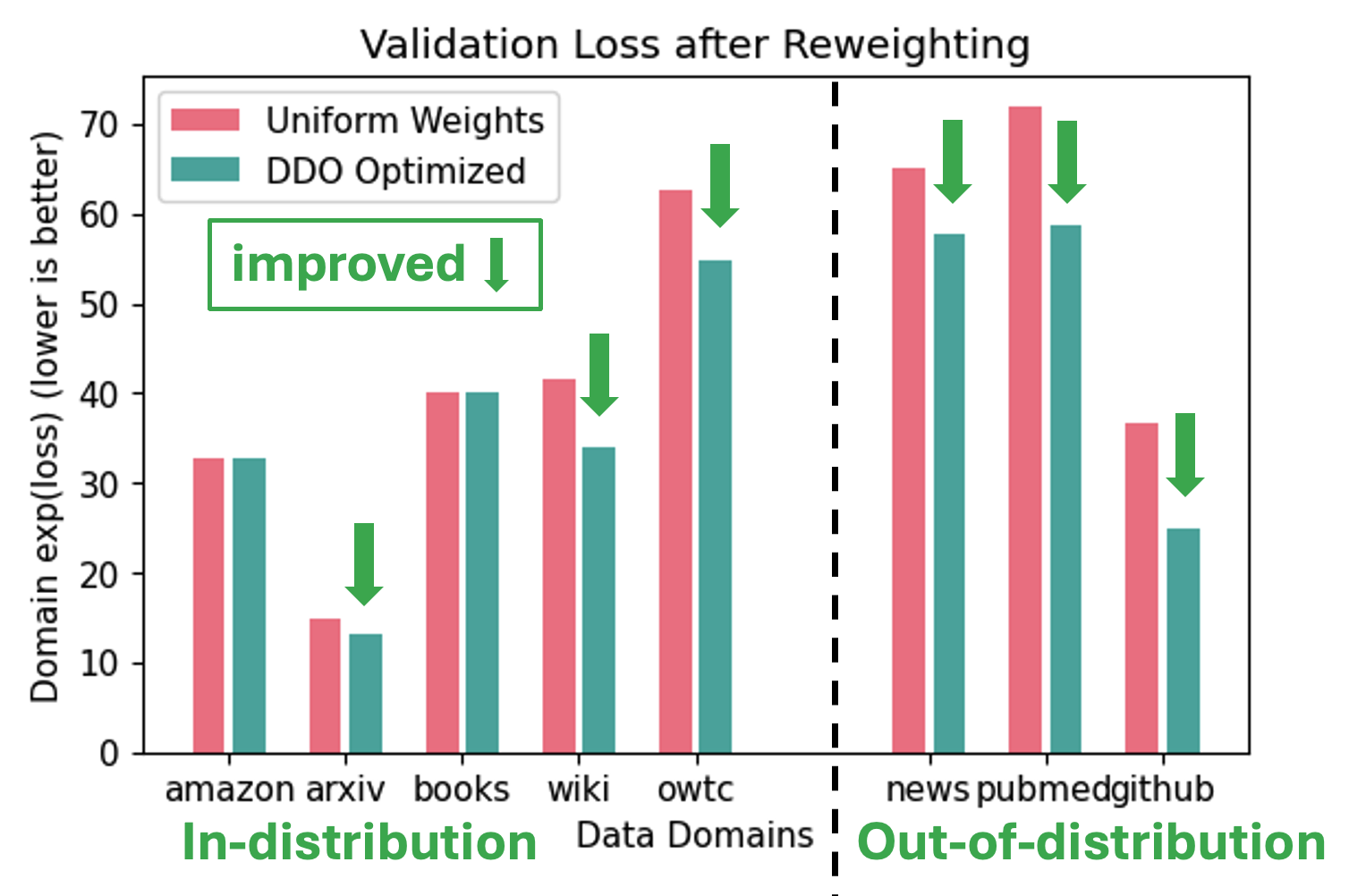}
        \caption{Validation Loss ($\downarrow$ lower is better)}
        \label{fig:figure2a}
    \end{subfigure}
    % \hfill
    \hspace{0em}
    \begin{subfigure}[b]{0.49\textwidth}
        \includegraphics[width=\textwidth]{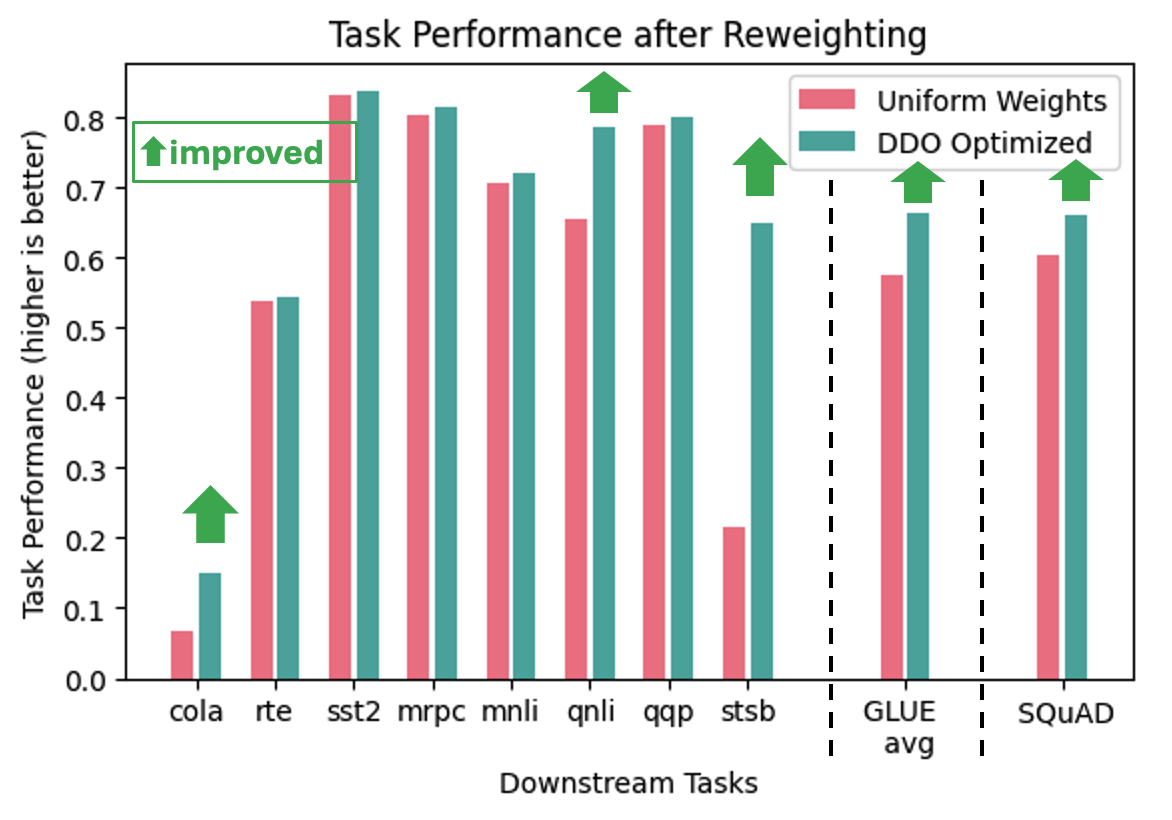}\vspace{-0.5em}
        \caption{Task Performance ($\uparrow$ higher is better)}
        \label{fig:figure2b}
    \end{subfigure} 
     \caption{\small{Optimizing domain weights with \textsc{DDO} algorithm for pre-training Encoder-only LMs (\texttt{BERT}). \textsc{DDO} substantially reduces validation loss. After reweighting, all training domains' loss has decreased or remained unchanged. Out-of-domain loss on non-training domains also decreased considerably. Enhanced performance is observed on all \texttt{GLUE} tasks (eval metric: \texttt{cola}: Matt. corr., \texttt{stsb}: Pearson corr., rest: acc.) and \texttt{SQuAD} (acc.).} \normalsize}
    \label{fig:figure2}\vspace{-1em}
\end{figure}

\paragraph{Analyzing DDO's effectiveness.} Recall that the key idea of DDO is to use a power-law--based parametric function 
to predict validation loss from domain weights. 
A major factor in DDO’s effectiveness lies in the \emph{accuracy} of this function. 
We quantify its predictive power via the \emph{average absolute relative error (AAR)} 
between the predicted and actual losses.
In our experiments, the AAR is $1.00\%$, 
indicating that DDO’s modeling closely reflects actual loss.

% \paragraph{Comparisons between DDO and RegMix}
% DDO and RegMix both aim to optimize domain weights at smaller, affordable data scales. Shown in Table \ref{tab:ddo_table1}, when fitted on the exact same pairs of domain weights and evaluation loss, \textit{\textcolor{teal}{DDO yields domain weights outperform RegMix's when training models at the same scale where domain weights were optimized, showing DDO-alone optimizes domain weights more accurately than RegMix.}} Further, RegMix does not consider adaptation for training models at different scales. Applying RegMix optimized weights on larger data scales appears less effective, evident by the \textit{widening gap} between its performance from DDO's. 

% \begin{remark}[\textbf{How accurately do the scaling-law functions approximate evaluation loss?}]
% The scaling law representation DDO employs accurately models the evaluation loss of the domain weights–with an \textit{\textcolor{teal}{Average Absolute Relative Error (AAR) $=1.00\%$}} between the predicted evaluation loss and measured value on the trained models. 
% \end{remark}

% Besides, the scaling law representation DDO employs accurately models the evaluation loss of the domain weights–with an impressive Average Absolute Relative Error (AAR) $=1.00\%$
% between the predicted evaluation loss and measured value on the trained models. 

\subsection{Evaluating~\autoscale}\label{sec:autoscale}
% \subsection{Research Question 2: Do DDO and optimal mix projection together yield significant overall gains?}

\paragraph{Effectiveness of our extrapolated weights.} Recall that \autoscale~is a two-stage pipeline: first, run DDO at smaller scales
to identify domain weights, then extrapolate those weights to a larger scale. We call the resulting allocation the \emph{\autoscale~weights}.
For GPT-2~Large, we run DDO on up to 0.6B tokens, then extrapolate to 3B and 10B tokens. Figure~\ref{fig:gpt10b} shows the change of test perplexity during training for models trained with 10B tokens
using \autoscale~weights versus baseline allocations.
\autoscale~consistently outperforms every baseline by a 28--38\% margin
and also demonstrates advantageous downstream performance.
Table~\ref{tab:gpt2_3b} demonstrates the results on 3B tokens,
revealing that \autoscale~maintains its superiority in both final loss achieved
and faster convergence.
Table~\ref{tab:gpt2_3b_domain} examines domain-wise test perplexities;
\autoscale~weights significantly reduce the loss on the \texttt{Books} domain
and improve worst-domain perplexity, also yielding a better average across domains.
Finally, Table~\ref{tab:gpt2_3b_task} evaluates eight downstream tasks.
The model trained with \autoscale~weights achieves the best overall performance,
further underscoring the effectiveness of our extrapolated domain weights.

For BERT, we train up to 288k steps (approximately 120\% of the original BERT-base budget \citep{devlin2018bert}).
Table~\ref{table10} shows that, compared to uniform (no reweighting), 
\autoscale~yields a 16.7\% speed-up at most data scales and a 10\% speed-up at the largest scale, 
demonstrating consistent effectiveness. 
However, these gains are smaller than those observed for GPT-2~Large, 
indicating that different architectures and training objectives may respond differently to domain reweighting.
This is also hinted at in Figure~\ref{fig:figure30}, 
where the evaluation loss shows a more uniform response to each domain,
suggesting fewer benefits from reweighting in BERT's setup.

\vspace{-0.5em}
\begin{figure}[h!]
    \centering
        \includegraphics[width=\textwidth]{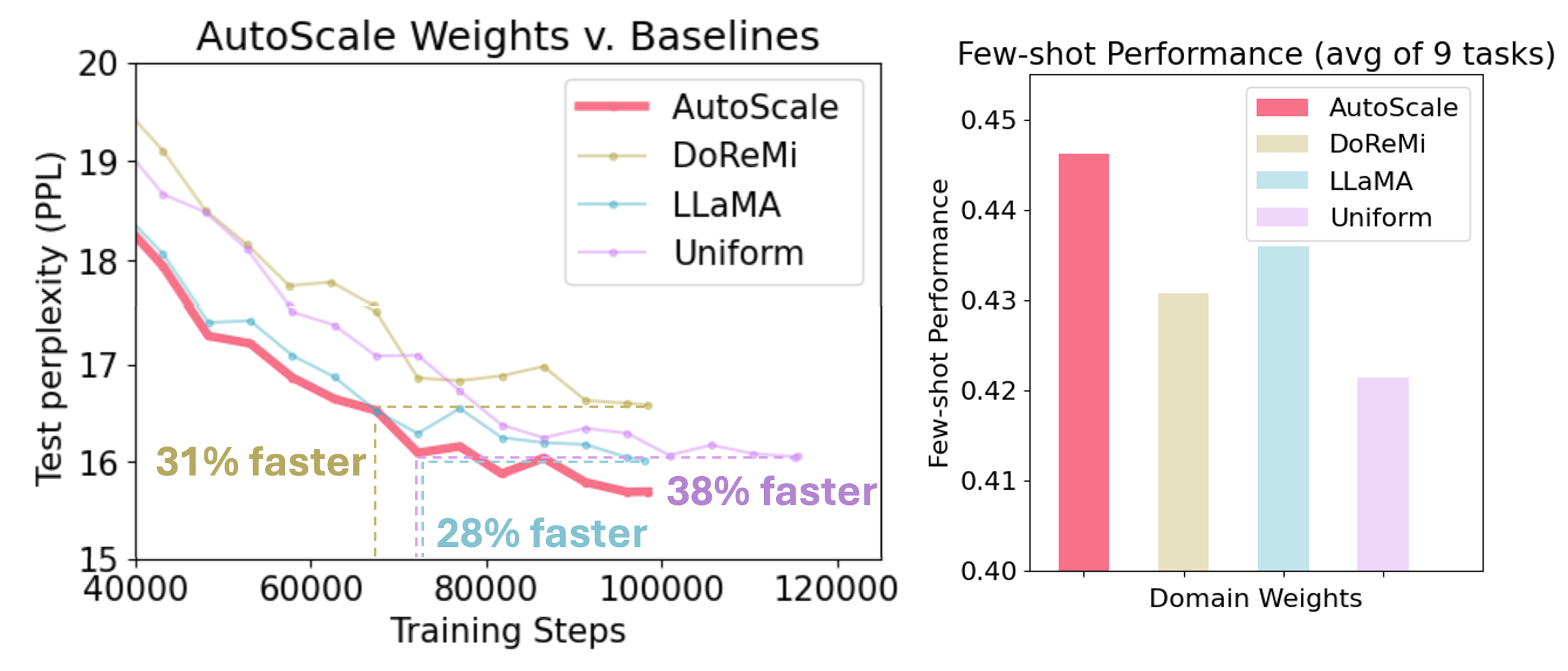}
        % \caption{Training Decoder-only LMs for 3B tokens.}
        \label{fig:figure7a}\vspace{-2em}
    \caption{\small{Training 774M Decoder-only LMs (GPT-2 Large) for 10B tokens (96k steps). \autoscale-predicted domain weights decrease test perplexity at least $28\%$ faster than any baseline with up to $38\%$ speed up, achieving best overall task performance.}}\normalsize
    \label{fig:gpt10b}\vspace{-0.5em}
\end{figure}

\begin{table}[h!]\centering\resizebox{0.7\linewidth}{!}{
\begin{tabular}{l|cc}
\toprule
\textbf{Weights} & Final Perplexity           & AutoScale Speed Improvement\\
(3B training tokens)&(PPL) & (\% steps saved to final PPL) \\ \midrule
AutoScale (ours)           &  \textbf{21.123}            & -    \\ \midrule
DoReMi          & 21.676            & 25\%     \\ 
Data Mixing Laws          & 23.333           & 37\%     \\ 
LLaMA           & 22.944            & 31\%   \\ RegMix           & 21.740            & 28\%   \\ 
\midrule
Uniform (30\% more tokens)           & 21.736           & 37\%     \\ 
\bottomrule
\end{tabular}}\vspace{-0.5em}
\caption{\small{Domain perplexity for 774M Decoder-only LMs (GPT-2 Large) trained for 3B tokens. \autoscale~-predicted weights decreases val loss at least $25\%$ faster than any baseline with up to $37\%$ speed up. Despite LLaMa weights being very different from uniform weights, they yield highly similar training efficiency at these data scales.}}\normalsize
\label{tab:gpt2_3b}\vspace{-0.5em}
\end{table}

\begin{table}[h!]\centering\resizebox{0.85\linewidth}{!}{
\begin{tabular}{l|c|cccc|c}
\toprule
\textbf{Domain/Method} & AutoScale            & DoReMi        & Data Mixing    & LLaMA & RegMix       & Uniform 
\\&(ours)&& Laws &&&(30\% more tokens) \\ \midrule
Common Crawl           & 25.598           & \textbf{24.116} & 30.824          & 21.464  &  24.430& 28.351                 \\
Github                 & 7.482           & 6.678          & \textbf{5.845} & 7.376  & 6.145 & 5.784                \\
Books                  & \textbf{29.162}  & 33.324          & 34.450          & 35.533  & 32.985 & 31.140                  \\
Wikipedia              & 18.828            & \textbf{17.154} & 26.795          & 21.110  & 20.177 & 19.570                  \\
C4                     & \textbf{34.242}  & 39.429          & 38.521          & 37.393  & 39.654 & 40.323                 \\
Stack Exchange         & 15.991           & 15.393          & \textbf{14.519} & 20.133  & 15.225& 13.890                 \\
Arxiv                  & 16.558           & 15.638          & \textbf{12.372} & 17.598  & 13.563 & 13.082                 \\ \midrule
\textbf{Average}       & \textbf{21.123} & 21.676         & 23.333          & 22.944 & 21.740 & 21.736                \\ \midrule
\textbf{Worst-domain}  & \textbf{34.242} & 39.429         & 38.521         & 37.393 & 39.654& 40.323                \\ \bottomrule
\end{tabular}}\vspace{-0.5em}\caption{\small{Domain perplexity for 774M GPT-2 Large trained for 3B tokens. \autoscale~ notably achieves the lowest average test perplexity while also significantly decreasing worse-domain perplexity. }}\normalsize\label{tab:gpt2_3b_domain}\vspace{-1.0em}
\end{table}

\begin{table}[h!]
\centering\vspace{-0.0em}
\resizebox{\linewidth}{!}{
\begin{tabular}{l|c|cccccccc}
\toprule
\textbf{Method/Task}                                                           & \textbf{Avg} & pubmedqa       & piqa            & hellaswag       & crows\_pairs & boolq           & arc\_easy      & truthfulqa  & hellaswag   \\
&&&&(10-shot)&\_english&&&\_mc2&(zero-shot)\\\midrule
\textbf{AutoScale (ours)}      & \textbf{0.4746}              & \textbf{0.536} & \textbf{0.6202} & \textbf{0.3021} & 0.5850                 & 0.6141 & \textbf{0.3977}  & 0.4385   & \textbf{0.3030}        \\ \midrule
Uniform Weights        & 0.4514                                                                  & 0.438          & 0.6115          & 0.2923          & 0.5886                & 0.5636          & 0.3742         & 0.4526    & 0.2907                     \\
LLaMA Weights            & 0.4585                                   & 0.492          & 0.6055          & 0.2944          & \textbf{0.5903}       & 0.5612          & 0.3956     & 0.434         & 0.2952                  \\

Data Mixing Laws        & 0.4610                                                   & 0.468          & 0.6061          & 0.2951          & 0.5778                & \textbf{0.6162}          & 0.3771       & \textbf{0.4537}     & 0.2938                  \\ 
DoReMi              & 0.4482                                       & 0.468          & 0.5985          & 0.2886          & 0.5742                & 0.5410          & 0.3750        & 0.4505       & 0.2896                     \\
RegMix & 0.4642  & 0.526 & 0.6077 & 0.2907 & 0.5850  & 0.6000 & 0.3721 & 0.4455 & 0.2868 \\
\bottomrule
\end{tabular}} \vspace{-0.5em}\caption{\small{Task performance for 774M GPT-2 Large trained for 3B tokens. Models trained with \autoscale~-predicted weights achieve the best overall performance across the tasks.}\normalsize}\label{tab:gpt2_3b_task}
\end{table}

% For 110M Encoder-only LMs, with \textsc{DDO}-optimized weights from models trained up to 0.5B tokens, we fit \autoscale~ predictor and use it to predict how the optimal domain weights will shift as we continue scaling up training data. Depicted in Fig.~\ref{fig:bert_additional}, the trends are similar to the pattern described above. We train 110M Encoder-only LMs with MLM for up to 288k steps ($\sim120\%$ of the pertaining data size for original \texttt{BERT-base} \citep{devlin2018bert}). Table \ref{table10} shows that, compared to without reweighting (uniform weights), \autoscale~\textit{-predicted weights speed up training by 16.7\% on most data scales with a 10\% speedup on the largest scale, validating its consistent effectiveness}. Nonetheless, the speedup is less impressive than in the results for Decoder-only LMs, demonstrating the different response to domain reweighting for models with different architecture or language modeling objectives. This is also hinted in Fig.~\ref{fig:figure30}(b), where the evaluation loss has a similar response to data from different domains, suggesting limited benefits from domain reweighting. 

\begin{table}[h!]\vspace{-0em}\centering{\resizebox{0.7\linewidth}{!}{
\begin{tabular}{lccccc}
\toprule
\textbf{Data Scale/steps} & 18k       & 36k       & 72k        & 144k       & 288k       \\\midrule
Final Loss (exp)  & 38.32     & 16.94     & 10.97      & 8.13       & 6.30       \\
Steps Saved       & 5k (28\%) & 5k (14\%) & 10k (14\%) & 20k (14\%) & 20k (10\%)\\
\bottomrule
\end{tabular}}}\vspace{-0.5em}\caption{\small{\autoscale~ notably improving training efficiency for \texttt{BERT} models on all scales–even for a considerably large scale, 288k steps, the speedup margin remains visible.}\normalsize}\label{table10} \vspace{-0.5em}
\end{table}
% Two sets of empirical studies are conducted: Causal Language Modeling (CLM) in Sec.~\ref{sec:exp-gpt}, and Masked Language Modeling (MLM) in Sec.~\ref{sec:exp-bert}. We train models with up to 10B tokens and report the number of steps saved to reach the same evaluation loss (perplexity). We also report downstream task performance to benchmark performance improvements after training the same number of steps.
% \call{@Yifan: help here.}

% \subsection{Experimental setup}
% %\paragraph{Models and Datasets}

% In Sec.~\ref{sec:exp-gpt}, we pretrain 774M Decoder-only LMs (\texttt{GPT-2 Large} architecture \citep{radford2019language}) \textbf{from scratch} on the \texttt{RedPajama} dataset \citep{together2023redpajama}. \texttt{RedPajama} dataset is an open-source reproduction of the training data used for LLaMA-1/2 models \citep{touvron2023llama}, totaling 1.2T tokens from 7 data domains with proportions: \texttt{Common Crawl} (67\%), \texttt{C4} \citep{raffel2020exploring} (15\%), \texttt{GitHub} (4.5\%), \texttt{Wikipedia} (4.5\%), \texttt{ArXiv} (2.5\%), and \texttt{StackExchange} (2.0\%). 
% In Sec.~\ref{sec:exp-bert}, we pretrain 110M Encoder-only LMs (\texttt{BERT-base} architecture \citep{DBLP:journals/corr/abs-1810-04805}) \textbf{from scratch} on data from 5 typical sources—\texttt{Amazon Reviews}, \texttt{Arxiv}, \texttt{Books}, \texttt{Wikipedia}, and \texttt{Open WebText Corpus} \citep{Gokaslan2019OpenWeb}. Further details are in App.~\ref{sec:appendix_exp_details_gpt} and \ref{sec:appendix_exp_details_bert}. Runtime and GPU hours are documented in App.~\ref{runtime}.

\paragraph{Examining how domain importance evolves with scale.} To illustrate the shift in domain importance, we first run DDO on GPT-2~Large 
across scales ranging from 30M to 1.2B tokens.
Figure~\ref{fig:figurex}(a) shows that the DDO-optimized weights differ \emph{visibly} 
at each scale, highlighting a clear shifting pattern.
Data sources with more standardized formats (\texttt{Wikipedia}, scientific papers)—often regarded as 
“high quality”—dominate at smaller scales but exhibit sharp diminishing returns as the data budget grows.
By contrast, domains with more diverse examples 
(\texttt{C4}, \texttt{CommonCrawl}) continue to lower training loss even at higher scales.

Consistently, taking DDO-optimized weights from up to 0.6B tokens, 
we use our theory to project how the composition would shift at scales beyond 1.2B.
Figures~\ref{fig:figurex}bd) and \ref{fig:figure4} show that as the training data scale grows, 
diverse domains (\texttt{C4}, \texttt{CommonCrawl}) command a larger share of the mix 
compared to “standard” domains. We observe a similar pattern with BERT, where we extrapolate the DDO-optimized weights at 0.5B tokens to even larger scales,
revealing that domains like \texttt{WebText} and \texttt{Amazon Reviews} gain significance over 
clean, standardized data (\texttt{Wikipedia}, \texttt{Arxiv}) 
(see Fig.~\ref{fig:bert_additional}).
A plausible explanation is that ``diverse'' data provides broader topical coverage and linguistic styles, 
while “standard” data saturates more quickly. 
% \begin{wrapfigure}{R}{0.4\textwidth}\vspace{-1em}
% \begin{minipage}{0.1\textwidth}
%     \scalebox{0.25}{\includegraphics{figs/auto_scales.png}}\vspace{-0.2em}
     
%     \label{fig:figure3}
% \end{minipage}\caption{ddo opt}\vspace{-1em}
% \end{wrapfigure}

\begin{figure}[h!] 
    \centering
    \begin{subfigure}[b]{0.42\textwidth}
        \includegraphics[width=\textwidth]{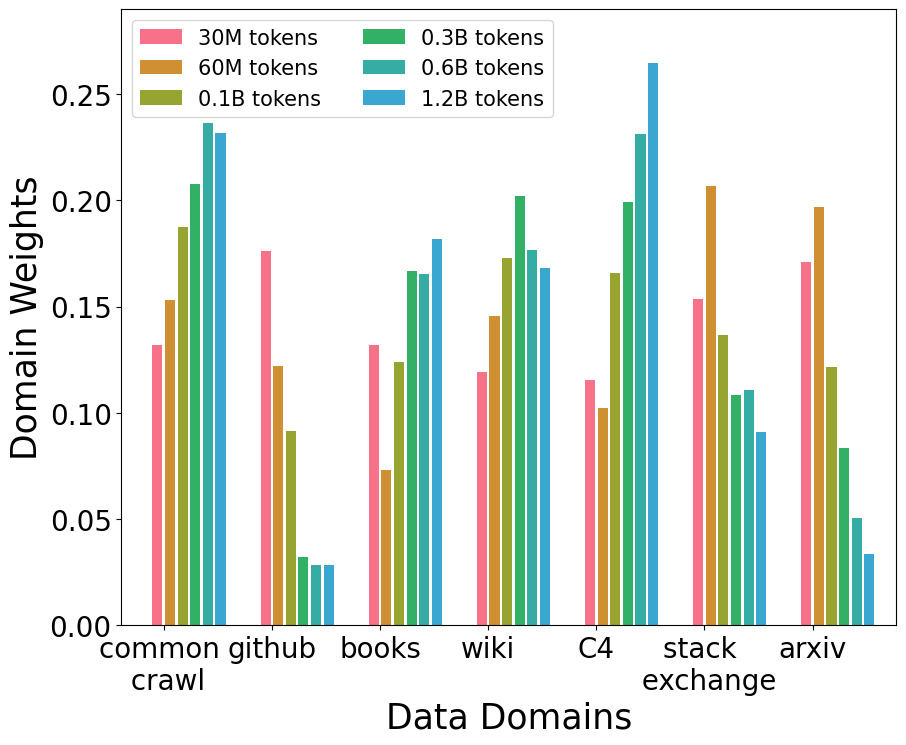}\vspace{-0.2em}
        \caption{DDO optimized domain weights.}
    \end{subfigure}
    % % \hfill
    % \hspace{-1em}
    \begin{subfigure}[b]{0.45\textwidth}
        \vspace{-0.5em}\includegraphics[width=\textwidth]{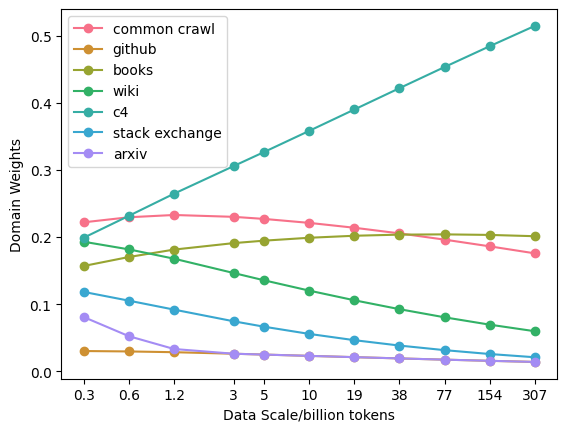}% \vspace{-0.5em}
        \caption{AutoScale projected domain weights.}
    \end{subfigure} \vspace{-0.5em}
     \caption{Domain importance evolves with training data scales. (GPT-2 Large)}
    \label{fig:figurex}\vspace{-0.5em}
\end{figure}

Note that these trends show \emph{how} our approach predicts domain importance may evolve, 
not a proof that each extrapolated mix \emph{guarantees} the best performance at its target scale. 
Nonetheless, the consistent shifting patterns across GPT-2~Large and BERT 
reinforce the idea that domain importance is \emph{scale-dependent}.

\vspace{-0.5em}\section{Conclusions} 
This paper explores how the importance of each training domain shifts across different scales 
and proposes a scale-aware framework (\autoscale) that outperforms existing approaches 
across various architectures, datasets, and training scales. 
Still, our experimental settings remain limited in scale and the diversity of evaluations. 
Extending this work to larger training budgets, additional data modalities, and broader benchmarks 
would clarify how well our insights generalize beyond the current scope. 
Another exciting next step is to adapt \autoscale~for 
directly optimizing downstream metrics, moving beyond perplexity 
as a rough proxy for language-model quality.

\section{Discussions}\vspace{-0.0em}
\paragraph{Online vs. offline data mixing methods.} 

These represent two independent paths to approach data mixing problems in LM pre-training. Current industrial practices are typically offline. Offline methods, such as DoReMi~\citep{xie2024doremi}, Doge~\citep{fan2023doge}, RegMix~\citep{liu2024regmix}, DML (Data Mixing Laws, \citet{ye2024data}), and our proposed AutoScale, decouple data curation from the final training process. \textit{This separation allows for the distribution of model development workloads across different teams (e.g., data, model, and training) working in parallel.} This approach is common in industrial settings for developing state-of-the-art LMs (e.g., Phi-4~\citep{abdin2024phi}, Llama4~\citep{meta2025llama}). Notably, many offline methods like DoReMi, DML, and BiMIX~\citep{ge2024bimix} originate from industrial research, whereas current online methods (ODM~\citep{albalak2023efficient}, Skill-it~\citep{chen2024skill}, Aioli~\citep{chen2024aioli}, ADO (Adaptive Data Optimization, \citet{jiang2024adaptive})) are primarily academic contributions.

Offline methods treat the training process integrally, primarily adjusting domain weights and training budgets before training commences. In contrast, online methods (e.g., ODM, Skill-it, Aioli, ADO) adjust domain weights dynamically during the training process. ODM, for instance, models online domain weighting as a bandit problem and updates domain weights based on local gradients. Considering the learning rate often changes throughout training, the loss landscape may be non-convex where a greedy approach cannot guarantee global optimality. Skill-It primarily considers fine-tuning scenarios, while Aioli extends this to small-scale LM pre-training. Aioli performs online estimation of the scaling law relationship between each pair of training and validation domains (in contrast to DML's offline estimation), while ADO considers scaling laws for individual training domains. The scaling law relationships these methods were originally proposed for adjusting training data size within a fixed training pipeline~\citep{kaplan2020scaling, hoffmann2022training}. They may not always accurately track a training process and its dynamics, especially with a varying learning rate, complicating their conceptual grounding. This creates space for future research contributions. For deployment in production pipelines, a better understanding on the predictability of the outcome under different training paradigms is desirable.

\paragraph{Effect of repeated documents and "data-constrained setting"}
The data constraint is relevant for real-world LM pre-training, particularly for low-resource domains. Beyond constraints on domain weights, practical limitations include the amount of available data and the maximum number of repetitions tolerated before significant performance degradation. Existing research on domain weighting has not explicitly modeled this factor. Despite this research gap, \citet{muennighoff2023scaling} offers positive evidence that current paradigms can still be effective under such constraints. The paper indicates that repeating data for up to 3 epochs has almost no negative impact, and models can still train reasonably well with up to 7 epochs of repetition. This suggests that in practice, repeating data from some domains a few times can be acceptable without significant detriment to performance. Current data mixing methods may be applied to these scenarios without major modifications.

\paragraph{Data mixing for model-size scaling}

Applying the data-scaling methodology to model-size scaling appears plausible. Scaling law papers~\citep{kaplan2020scaling, hoffmann2022training} reported similar scaling trends between validation loss and training data budget/model size. \citep{shukor2025scaling} proposes to model the combined effects of data and model-size scaling as a linear addition and validates its effectiveness in optimizing domain weights for LLM, native multimodal model (NMM), and large vision models (LVM) pretraining.

Additional discussions on \textbf{Intuition for Stage 2 of \autoscale, Optimal mix projection} and \textbf{Number of predictors (scaling law components)} can be found in Appendix \ref{app:diss}.

\section*{Impact Statement}

Reducing the complexity and resource requirements associated with pretraining LLMs, \autoscale~ contributes to the democratization of AI. Smaller organizations, academic institutions, and individual researchers can more easily participate in cutting-edge AI research and development, fostering innovation and collaboration across the AI community.  Moreover, learning from massive amounts of data requires large and costly computational resources, which not only consume substantial energy but also generate a significant carbon footprint, contributing to environmental issues. Furthermore, these resources quickly become obsolete due to the rapid pace of technological advancements, leading to e-waste. This research makes contributions to mitigating these issues by improving the efficiency of resource utilization in AI training.

\section*{Acknowledgement}
This work is supported in part by the National Science Foundation under grants IIS-2312794, IIS2313130, OAC-2239622, Amazon-Virginia Tech Initiative in Efficient and Robust Machine Learning, AWS computational credits, and the Commonwealth Cyber Initiative. The authors are grateful for Ankit Battawar and Alix Delgado from AWS, whose dedicated help and support were crucial for securing computing resources and implementing empirical studies.

\clearpage
\newpage

\bibliography{main_paper}
\bibliographystyle{icml2025}

%%%%%%%%%%%%%%%%%%%%%%%%%%%%%%%%%%%%%%%%%%%%%%%%%%%%%%%%%%%%%%%%%%%%%%%%%%%%%%%
%%%%%%%%%%%%%%%%%%%%%%%%%%%%%%%%%%%%%%%%%%%%%%%%%%%%%%%%%%%%%%%%%%%%%%%%%%%%%%%
% APPENDIX
%%%%%%%%%%%%%%%%%%%%%%%%%%%%%%%%%%%%%%%%%%%%%%%%%%%%%%%%%%%%%%%%%%%%%%%%%%%%%%%
%%%%%%%%%%%%%%%%%%%%%%%%%%%%%%%%%%%%%%%%%%%%%%%%%%%%%%%%%%%%%%%%%%%%%%%%%%%%%%%

\newpage
\onecolumn

\begin{appendices}

% 附录目录生成
\startcontents[appendices]
% \appendixpage
\printcontents[appendices]{}{1}{\setcounter{tocdepth}{2}}

\clearpage

\section{Algorithms and Operational Pipeline}
\label{app:operation}

\begin{center}
\scalebox{1.0}{\begin{minipage}{1.0\textwidth}\vspace{-1.5em}
      \begin{algorithm}[H]
\caption{Direct Data Optimization (\textsc{DDO})}\label{alg:ddo}
\begin{algorithmic}
\Require $m$ domains (data sources) with data $D_1 \dots D_m$, data budget $N_0$ ($\ll$ for full-scale training), training dataset $S$, model parameters $\boldsymbol{\theta}$, validation loss $\mathcal{L}_v$, perturbation ratio $r>1$ (e.g., $r=3$).
\State Initialize weights for all domains $\forall i\in \{1,\dots m\}$: $w_i \gets 1/m$;
\State Initialize training data for all domains $\forall i\in \{1,\dots m\}$: sample $S_i \subset D_i$ where $|S_i|=w_i\cdot N$;
\State Train the model on data $S=\{S_1\dots S_m\}$ and evaluate its loss $\mathcal{L}_v^0 \gets \mathcal{L}_v(\boldsymbol{\theta}^*(S))$;
\For{$j$ from $1$ to $m$}
\State $w_j^+ \gets r\cdot w_j$; \Comment{Perturb domain weights (+)}
\State Resample $S_j^+ \subset D_j$ where $|S_j^+|=w_j^+\cdot N$;
\State Train the model on data $S=(\{S_1\dots S_m\}\setminus S_j)\cup S_j^+$ and evaluate its loss $\mathcal{L}_j^+ \gets \mathcal{L}_v(\boldsymbol{\theta}^*(S))$;
\State $w_j^- \gets \frac{1}{r}\cdot w_j$; \Comment{Perturb domain weights (-)}
\State Resample $S_j^- \subset D_j$ where $|S_j^-|=w_j^-\cdot N$;
\State Train the model on data $S=(\{S_1\dots S_m\}\setminus S_j)\cup S_j^-$ and evaluate its loss $\mathcal{L}_j^- \gets \mathcal{L}_v(\boldsymbol{\theta}^*(S))$;
\State OLS fit for scaling functions $N_0^i,\gamma_i,\ell_i = \arg\min_{N_0^i,\gamma_i,\ell_i} [\mathcal{L}_v^0-(N_0^i + N_i)^{-\gamma_i}-\ell_i]^2+[\mathcal{L}_{(+i)}-(N_0^i + N_i^+)^{-\gamma_i}-\ell_i]^2+[\mathcal{L}_{(-i)}-(N_0^i + N_i^-)^{-\gamma_i}-\ell_i]^2$;
\EndFor
\State Output optimized domain weights $\mathbf{w^*} =\arg\min_{\mathbf{w'}\in \mathbb{W}^m}\sum_{i=1}^m (N_0^i + w_i'\cdot N)^{-\gamma_i}$.
\end{algorithmic}
\end{algorithm}
\end{minipage}}
\end{center}

\begin{center}
\scalebox{1.0}{\begin{minipage}{1.0\textwidth}
      \begin{algorithm}[H]
\caption{\autoscale~}\label{alg:as}
\begin{algorithmic}
\Require Optimal domain weights (obtained from \textsc{DDO}) $\mathbf{w^{(1)*}}$at data scale $N^{(1)}$ and $\mathbf{w^{(2)*}}$ at data scale $N^{(2)}$, target data scale $N^{(t)}$, where $N^{(1)}<N^{(2)}<N^{(t)}$; resolution $\delta$.

\State Optimal domain data $\mathbf{N^*}(N^{(1)}) \gets \mathbf{w^{(1)*}}\cdot N^{(1)}$;
\State Optimal domain data $\mathbf{N^*}(N^{(2)}) \gets \mathbf{w^{(2)*}}\cdot N^{(2)}$;
\State Current data budget $N \gets \sum_i N_i^{(2)*}$;
\State Optimal domain data under current data budget $\mathbf{N^*}(N) \gets \mathbf{N^*}(N^{(2)})$;
\While{$N < N^{(t)} $}
\State Compute optimal domain data under the next data budget: $\mathbf{N^*}(N^{\text{next}}) \gets \mathbf{N^*}(N)[(\mathbf{N^*}(N^{(1)}))^{-1}\mathbf{N^*}(N^{(2)})]^\delta$;
\State Compute the next data budget $N^{\text{next}} \gets \sum_i N_i^{*}$;
\State Update current data budget $N \gets N^{\text{next}}$;
\EndWhile
\State Output predicted optimal domain weights: $\mathbf{\hat{w}^{(t)*}} \gets \mathbf{N^*}(N)/N$.
\end{algorithmic}
\end{algorithm}
\end{minipage}}
\end{center}

% \begin{figure}[h!]
% \begin{center}
%   \includegraphics[width=1.0\textwidth]{figs/main_fig_crop.png}
%   % \vspace{-1em}
%   \caption{LLMs are pre-trained using data from different sources or domains, yet determining the optimal data composition is challenging. We propose \autoscale~, an automated tool that finds a compute-optimal data composition for training at any desired target scale. \autoscale~ first determines the optimal composition at a small scale using a novel bi-level optimization framework, \underline{D}irect \underline{D}ata \underline{O}ptimization (\textsc{DDO}), and then fits a predictor to estimate the optimal composition at larger scales. The predictor's design is inspired by our theoretical analysis of scaling laws related to data composition, which could be of independent interest. In empirical studies, \autoscale~ decreases validation perplexity at least 25\% faster than any baseline with up to 38\% speed up compared to without reweighting, achieving the best overall performance across downstream tasks. 
%   }\label{fig:example}% \vspace{-2em}
%   \end{center}
% \end{figure}% \vspace{-1em}

\textbf{Operational Pipeline (\textsc{DDO})}
\begin{enumerate}
\item Train a base proxy model with uniform weights (or reference weights, if available);
\item At each time, add/reduce data quantity for one domain and re-train the proxy model;
\item Fit power law scaling functions and solve the optimization problem;
\item Iterate the process if necessary.
\end{enumerate}

\clearpage

\textbf{Operational Pipeline (\autoscale~)}
\begin{enumerate}

\item For two smaller training data scales $N^{(1)}$ and $N^{(2)}$ where re-training the model is affordable, find their corresponding optimal training data compositions $\mathbf{N^*}(N^{(1)})$ and $\mathbf{N^*}(N^{(2)})$ using \textsc{DDO} Algorithm described above;
\item Initialize current data budget at $N=N^{(2)}$;
\item With the chosen resolution $\delta$, predict the next optimal training data composition as $\mathbf{N^*}(N^{\text{next}}) = \mathbf{N^*}(N)[(\mathbf{N^*}(N^{(1)}))^{-1}\mathbf{N^*}(N^{(2)})]^\delta$, yielding optimal domain weights $w_i^{*}=N_i^{*}(N^{\text{next}})/N^{\text{next}}$ at new training data scale $N^{\text{next}}=\sum_i N_i^{*}(N^{\text{next}})$;
\item Update current data budget to $N=N^{\text{next}}$. Repeat this process until the target training data scale is reached.
\end{enumerate}

\clearpage

\section{Proofs for Section~\ref{sec:thm}, Optimal mix projection}
\label{sec:appendix_proofs}

\subsection{Theorem 1:  Scale-Dependent Optimal Composition}\label{app:lemma1}

\begin{theorem} [Scaling Law for Optimal Data Compositions \textbf{(restated)}]
    Consider the following optimization problem
    \[
        \min_{\mathbf{N}} \left\{ \sum_{i=1}^m \beta_i N_i^{-\gamma_i} \Bigg| \sum_{i=1}^m N_i = N \right\}.
    \]
    For any two compute budgets $N^{(1)} \neq N^{(2)}$, let $\mathbf{N}^*(N^{(1)})$ and $\mathbf{N}^*(N^{(2)})$ be their respective minimizers. 
    For any third data composition $\mathbf{N}(N^{(3)})$, if there exists some constant $k\in\mathbb{R}^+$ such that
    \[
    \mathbf{N}(N^{(3)}) = \mathbf{N}^*(N^{(2)})[(\mathbf{N}^*(N^{(1)})) ^{-1} \mathbf{N}^*(N^{(2)})]^k, 
    \]
    then, $\mathbf{N}(N^{(3)})$ is the minimizer for data budget $N^{(3)}=\sum_{i=1}^m N^{(3)}_i$, given as 
         \[
        \mathbf{N}(N^{(3)}) = \arg\min_{\mathbf{N}} \left\{ \sum_{i=1}^m \beta_i N_i^{-\gamma_i} \Bigg| \sum_{i=1}^m N_i = N^{(3)} \right\}= \mathbf{N}^*(N^{(3)}).
    \]
    % Then, for any third compute budget $N^{(3)}$ such that $N^{(1)} \neq N^{(3)} \neq N^{(2)}$, the corresponding minimizer $\mathbf{N}^*(N^{(3)})$ must satisfy $ \mathbf{N}^*(N^{(3)}) = \mathbf{N}^*(N^{(2)})(\mathbf{N}^*(N^{(1)})) ^{-1} \mathbf{N}^*(N^{(2)})$.
\end{theorem}

% \ruoxi{a cleaned version of proof with more uniform notations:} 
\begin{proof} \textbf{Setup:}  
We begin with the following optimization problem, defined at a given total training data scale \(N^{(1)}\):  
\[
\min_{\mathbf{N}} \left\{ \sum_{i=1}^m \beta_i N_i^{-\gamma_i} \,\middle|\, \sum_{i=1}^m N_i = N^{(1)} \right\}.
\]  
Here, \(\mathbf{N} = \mathrm{diag}\{N_1, N_2, \ldots, N_m\}\) is a diagonal matrix whose diagonal entries are the amounts of data allocated to each of the \(m\) domains.

For this problem, there exists a unique optimal solution \(\mathbf{N}^*(N^{(1)}) = \mathrm{diag}\{N_1^{(1)*}, N_2^{(1)*}, \ldots, N_m^{(1)*}\}\). This \(\mathbf{N}^*(N^{(1)})\) represents the compute-optimal data composition at the data scale \(N^{(1)}\).

\textbf{First-Order Conditions (KKT):}
At optimality, the Karush–Kuhn–Tucker (KKT) conditions ensure that the partial derivatives of the objective function with respect to each \(N_i\) are equal (up to the same Lagrange multiplier for the equality constraint \(\sum_i N_i = N^{(1)}\)). For any pair of domains \(a\) and \(b\), we must have:
\[
\left.\frac{\partial}{\partial N_a}\left(\sum_{i=1}^m \beta_i N_i^{-\gamma_i}\right)\right|_{N_a = N_a^{(1)*}}
= \left.\frac{\partial}{\partial N_b}\left(\sum_{i=1}^m \beta_i N_i^{-\gamma_i}\right)\right|_{N_b = N_b^{(1)*}}.
\]

Computing these derivatives, we get:
\[
-\beta_a \gamma_a (N_a^{(1)*})^{-\gamma_a - 1} = -\beta_b \gamma_b (N_b^{(1)*})^{-\gamma_b - 1}.
\]

From this equality:
\[
\frac{\beta_a \gamma_a}{\beta_b \gamma_b} 
= \frac{(N_a^{(1)*})^{\gamma_a + 1}}{(N_b^{(1)*})^{\gamma_b + 1}}.
\]

Rearranging, we obtain a fundamental scaling relationship:
\[
N_a^{(1)*} = \left(\frac{\beta_a \gamma_a}{\beta_b \gamma_b}(N_b^{(1)*})^{\gamma_b + 1}\right)^{\frac{1}{\gamma_a + 1}}.
\]

\textbf{Scaling to a Second Data Scale \(N^{(2)}\):}
Now consider a different total data scale \(N^{(2)} \neq N^{(1)}\), with the corresponding compute-optimal solution \(\mathbf{N}^*(N^{(2)}) = \mathrm{diag}\{N_1^{(2)*}, N_2^{(2)*}, \ldots, N_m^{(2)*}\}\).

Suppose we know how the optimal amount of data for domain \(b\) changes from \(N^{(1)}\) to \(N^{(2)}\). Specifically, let:
\[
N_b^{(2)*} = m \cdot N_b^{(1)*}
\]
for some scaling factor \(m > 0\).

Applying the same relationship used for the first scale, but now at the second scale, we find that for domain \(a\):
\[
N_a^{(2)*} = \left(\frac{\beta_a \gamma_a}{\beta_b \gamma_b} (N_b^{(2)*})^{\gamma_b + 1}\right)^{\frac{1}{\gamma_a + 1}}
= \left(\frac{\beta_a \gamma_a}{\beta_b \gamma_b} (m \cdot N_b^{(1)*})^{\gamma_b + 1}\right)^{\frac{1}{\gamma_a + 1}}.
\]

This simplifies to:
\begin{equation}\label{eq:kktna}
N_a^{(2)*} = m^{\frac{\gamma_b + 1}{\gamma_a + 1}} N_a^{(1)*}.
\end{equation}

Notice that \(m^{\frac{\gamma_b + 1}{\gamma_a + 1}} \neq m\) in general. Thus, when the budget scales by a factor \(m\) in domain \(b\), the optimal amount for domain \(a\) scales by a different factor. This shows that the optimal composition is scale-dependent.

\textbf{Predicting a Third Scale \(N^{(3)}\):}
We now know the optimal compositions at two scales \(N^{(1)}\) and \(N^{(2)}\). Consider a third scale \(N^{(3)}\) and its optimal solution \(\mathbf{N}^*(N^{(3)}) = \mathrm{diag}\{N_1^{(3)*}, N_2^{(3)*}, \ldots, N_m^{(3)*}\}\).

If we choose \(N_b^{(3)*}\) such that:
\begin{equation}\label{eq:pf-frac}
    \frac{N_b^{(3)*}}{N_b^{(2)*}} = \frac{N_b^{(2)*}}{N_b^{(1)*}},
\end{equation}
then the change in \(N_b\) from \(N^{(2)}\) to \(N^{(3)}\) mirrors the change from \(N^{(1)}\) to \(N^{(2)}\).

Since the scaling exponent \(\frac{\gamma_b + 1}{\gamma_a + 1}\) remains the same, this symmetrical setup leads to:
\[
N_a^{(3)*} = \frac{(N_a^{(2)*})^2}{N_a^{(1)*}}.
\]

\textbf{Matrix Form:}
Because all domains scale in a similar fashion, we can write this relationship compactly using diagonal matrices. Define:
\[
\mathbf{N^*}(N^{(i)}) = \mathrm{diag}\{N_1^{(i)*}, N_2^{(i)*}, \ldots, N_m^{(i)*}\}.
\]

The element-wise relationship \(\frac{(N_a^{(2)*})^2}{N_a^{(1)*}}\) for each domain \(a\) can be expressed as:
\[
\mathbf{N}^*(N^{(3)}) = \mathbf{N}^*(N^{(2)})(\mathbf{N}^*(N^{(1)}))^{-1}\mathbf{N}^*(N^{(2)}).
\]

Here, \((\mathbf{N}^*(N^{(1)}))^{-1}\) is the inverse of the diagonal matrix \(\mathbf{N}^*(N^{(1)})\), obtained by taking the reciprocal of each diagonal element \(N_a^{(1)*}\).

We have shown that given two distinct data scales \(N^{(1)}\) and \(N^{(2)}\) and their corresponding optimal solutions \(\mathbf{N}^*(N^{(1)})\) and \(\mathbf{N}^*(N^{(2)})\), one can construct a third optimal solution \(\mathbf{N}^*(N^{(3)})\) using the formula:
\[
\mathbf{N}^*(N^{(3)}) = \mathbf{N}^*(N^{(2)})(\mathbf{N}^*(N^{(1)}))^{-1}\mathbf{N}^*(N^{(2)}).
\]

This relationship holds without needing to explicitly estimate the parameters \(\gamma_i\) or \(\beta_i\), and it confirms that the optimal data composition is scale-dependent. Thus, the given scaling law for optimal data compositions is established.

\textbf{Generalization to Prediction for Any Data Scale:} Finally, we generalize from the case in Eq.~\eqref{eq:pf-frac} to allow prediction of optimal data composition for \textit{any} data scale. Consider for some constant $\forall k\in\mathbb{R}^+$, we choose $N_b^{(k)*}$ such that 
\begin{equation*}
    \frac{N_b^{(k)*}}{N_b^{(2)*}} = \left(\frac{N_b^{(2)*}}{N_b^{(1)*}}\right)^k=m^k,
\end{equation*}

Same as the procedure in Eq.~\eqref{eq:kktna}, KKT optimality conditions yield the corresponding optimal data quantity for domain a at the same scale as $N_b^{(k)*}$ as
\begin{equation*}
N_a^{(k)*} = (m^k)^{\frac{\gamma_b + 1}{\gamma_a + 1}} N_a^{(2)*}= \left(\frac{N_a^{(2)*}}{N_a^{(1)*}}\right)^k N_a^{(2)*}.
\end{equation*}
Rearranging in the matrix form, we have the following formula
\[
\mathbf{N}^*(N^{(k)}) = \mathbf{N}^*(N^{(2)})[\mathbf{N}^*(N^{(2)})(\mathbf{N}^*(N^{(1)}))^{-1}]^k,
\] 
which concludes the proof.

\textbf{Application in \textsc{AutoScale:}} Note that this formulate holds for any $k\in \mathbb{R}^+$. Thus, by scanning through the values of $k$, one can find optimal data composition $\mathbf{N}^*$ for any target data scale $N=\sum_i N^*_i$. \textit{In practice, for a target data scale $N>N^{(2)}=\sum_i N^{(2)*}_i$, one only needs to conduct a line search along $k>1$ to find the value of $k$ where $\sum_i N^{(k)*}_i=N$ to determine its corresponding optimal data composition $\mathbf{N}^*$.}
\end{proof}
% this means

% **Conclusion:**
% We have shown that given two distinct data scales \(N^{(1)}\) and \(N^{(2)}\) and their corresponding optimal solutions \(\mathbf{N}^*(N^{(1)})\) and \(\mathbf{N}^*(N^{(2)})\), one can construct a third optimal solution \(\mathbf{N}^*(N^{(3)})\) using the formula:
% \[
% \mathbf{N}^*(N^{(3)}) = \mathbf{N}^*(N^{(2)})(\mathbf{N}^*(N^{(1)}))^{-1}\mathbf{N}^*(N^{(2)}).
% \]

% This relationship holds without needing to explicitly estimate the parameters \(\gamma_i\) or \(\beta_i\), and it confirms that the optimal data composition is scale-dependent. Thus, the given scaling law for optimal data compositions is established.

% \[
% \mathbf{N}^*(N^{(3)}) = \mathbf{N}^*(N^{(2)})(\mathbf{N}^*(N^{(1)}))^{-1}\mathbf{N}^*(N^{(2)}).
% \]

% conducting a line search over k
% optimal data composition 
% N* can be found for any 
% data budegt N=\sum_i

% ---
% **End of Revised Proof**}

\begin{remark} [An example]
    This example helps visualize the operation pipeline. 
    
    If at training data scale $N^{(1)}=N_a^{(1)}+N_b^{(1)}=200$, we have optimal domain data composition as $N_a^{(1)*}=100, N_b^{(1)*}=100$ ($50\%-50\%$); and at scale $N^{(2)}=N_a^{(2)}+N_b^{(2)}=500$, we have optimal domain data composition as $N_a^{(2)*}=300, N_b^{(2)*}=200$ ($60\%-40\%$). Then, from the theorem, when the optimal domain data composition has $N_a^{(3)*}=(N_a^{(2)*})^2/N_a^{(1)*}=900$, we can predict $N_b^{(3)*}=(N_b^{(2)*})^2/N_b^{(1)*}=400$, which gives the optimal ratio at $N^{(3)}=N_a^{(3)}+N_b^{(3)}=1300$ as $69\%-31\%$. 
    
    Similarly, 
    
\small 
% For $N_a^{(1)*}=100$, we have $N_b^{(1)*}=100$, which gives the optimal ratio at $N^{(1)}=200$ as $50\%-50\%$\\
% For $N_a^{(2)*}=300$, we have $N_b^{(2)*}=200$, which gives the optimal ratio at $N^{(2)}=500$ as $60\%-40\%$\\
% For $N_a^{(3)*}=900$, we have $N_b^{(3)*}=400$, which gives the optimal ratio at $N^{(3)}=1300$ as $69\%-31\%$\\
For $N_a^{(4)*}=2700$, we have $N_b^{(4)*}=800$, which gives the optimal ratio at $N^{(4)}=3500$ as $77\%-23\%$\\
For $N_a^{(5)*}=8100$, we have $N_b^{(5)*}=1600$, which gives the optimal ratio at $N^{(5)}=9700$ as $84\%-16\%$\\
For $N_a^{(6)*}=24300$, we have $N_b^{(6)*}=3200$, which gives the optimal ratio at $N^{(6)}=27500$ as $88\%-12\%$\\
For $N_a^{(7)*}=72900$, we have $N_b^{(7)*}=6400$, which gives the optimal ratio at $N^{(7)}=79300$ as $92\%-8\%$\\
For $N_a^{(8)*}=218700$, we have $N_b^{(8)*}=12800$, which gives the optimal ratio at $N^{(8)}=231500$ as $94\%-6\%$\\
For $N_a^{(9)*}=656100$, we have $N_b^{(9)*}=25600$, which gives the optimal ratio at $N^{(9)}=681700$ as $96\%-4\%$\normalsize

We visualize it in Fig.~\ref{fig:example}.
\begin{figure}[h!]
\begin{center}
  \includegraphics[width=0.7\textwidth]{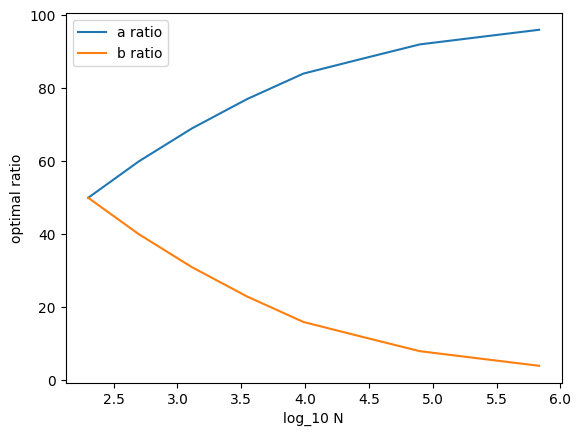}
  \vspace{-1em}
  \caption{Illustration: optimal data composition scales in exponential-style functions with training data quantity.
  }\label{fig:example}
  \vspace{-1em}
  \end{center}
\end{figure}% \vspace{-1em}
\end{remark}

\subsection{Scaling Latent Skills}\label{app:theorem1}
We extend this theory to a general case where the evaluation loss is the perplexity averaged over training domains.
Consider the evaluation is composed of a number of \textit{independent} sub-tasks ("latent skills" \citep{tiong2024toward}) which are hidden variables, where each of them observes a power law scaling law relationship with the amount of data contributing to this task ("equivalent data size"), $
\mathcal{L} =\ell_0+\beta_a\cdot K_a^{-\gamma_a}+\beta_b\cdot K_b^{-\gamma_b}+\beta_c\cdot K_c^{-\gamma_c} + \cdots
$
where scalar $K_j\geq 0$ denote equivalent data size for \textit{skill$_j$}, and constants $(\beta_j, \gamma_j)\geq 0$ are coefficients associated with \textit{skill$_j$}, respectively. Mathematically, these latent skills can be seen as an orthogonal basis that spans the space of evaluation loss.

Consider training data from each domain $D_i$ contributes to these skills to varying degrees, where Equivalent data size for \textit{skill$_j$}, $K_j$, is given as
$
K_j = c_{j,1}\cdot N_1 + c_{j,2}\cdot N_2 + \cdots
$
where $N_i=w_i\cdot N$ denotes the amount of training data from domain $D_i$ and constant $c_{j,i}$ is the coefficient measuring the degree of contribution between domain $D_i$ and \textit{skill$_j$}.
Defining diagonal matrices for training data composition $\mathbf{N}=diag\{N_1,N_2, \cdots\}$ and skill data composition $\mathbf{K}=diag\{K_a,K_b,\cdots\}$, we have 
$
\mathbf{K} = \mathbf{A}\mathbf{N}
$,
where $\mathbf{A}_{ji}=c_{j,i}$ is the matrix for coefficients. 
For simplicity, we consider training data from each domain will be \textit{distributed} to the skills such that $\forall i, \sum_j N_i = 1 $. This gives the amount of total training data from all domains is identical to the amount of total equivalent data for all skills, $\sum_j K_j = \sum_i N_i$.
For a training data scale $N=\sum_i N_i=\sum_j K_j$, define optimal skill data composition $\mathbf{K^*}=diag\{K_a^*, K_b^*, \cdots\}$ as the minimizer of $\mathcal{L}$, given as
$
    \mathbf{K^*} = {\arg\min}_{\sum_j K_j=N} \ell_0+\beta_a\cdot K_a^{-\gamma_a}+\beta_b\cdot K_b^{-\gamma_b}+ \cdots
$. 
Theoretically, there can be an infinite number of latent skills. For analysis, we consider a finite number of $k$ independent skills \textit{most important} for the evaluation. This can considered as performing Principal Components Analysis (PCA) with orthogonal transformation and selecting the first $k$ independent components. We consider the standard scenario with an equal number of relevant skills and data domains where $k=m$ and $\mathbf{A}$ is a square matrix with full rank. This describes the case where this optimization problem is well-defined. We discuss in App.~\ref{app:theorem1} what will happen in other scenarios. 
In this case, $\mathbf{A}$ is invertible and the corresponding optimal training data composition for $\mathbf{K^*}$ can be given as
$
\mathbf{N^*} =  \mathbf{A}^{-1}\mathbf{K^*}
$.

We provide the following theorem, which states that for the scenario described above, optimal training data composition scales in exponential-style functions with training data quantity and can be directly predictable from that of smaller scales \textit{without needing to identify the latent skills}.

% \renewcommand{\thetheorem}{2}
% \begin{theorem} [Scaling Latent Skills]
% Consider the evaluation is composed of a number of \textit{independent} sub-tasks ("latent skills") where each of them observes a power law scaling law relationship with the amount of data contributing to this task ("equivalent data size")as described above. 
% If we have optimal data compositions $\mathbf{N^*}=diag\{N_a^*, N_b^*, \cdots\}$ 
% where its corresponding skill data composition $\mathbf{K^*}=diag\{K_a^*, K_b^*, \cdots\}=\mathbf{A}\mathbf{N^*}$ minimizes $\mathcal{L}$ s.t. $\sum_j K_j=\sum_i N^*=N$, and $\mathbf{N^{'*}}=diag\{N_a^{'*}, N_b^{'*}, \cdots\}$ 
% where its corresponding skill data composition $\mathbf{K^{'*}}=diag\{K_a^{'*}, K_b^{'*}, \cdots\}=\mathbf{A}\mathbf{N^{'*}}$ minimizes $\mathcal{L}$ s.t. $\sum_j K_j'=\sum_i N^{'*}=N'$ where $N'\neq N$, then, other optimal data compositions $\mathbf{N^{''*}}=diag\{N_a^{''*}, N_b^{''*}, \cdots\}$ 
% where its corresponding skill data composition $\mathbf{K^{''*}}=diag\{K_a^{''*}, K_b^{''*}, \cdots\}=\mathbf{A}\mathbf{N^{''*}}$ minimizes $\mathcal{L}$ s.t. $\sum_j K_j''=\sum_i N^{''*}=N''$ where $N''\neq N'\neq N$ can be given as
% $
% \mathbf{N''} = (\mathbf{N'})^{-1}\mathbf{N}\mathbf{N'}
% $.
% \end{theorem}
%     See App.~\ref{app:theorem1} for the formal statement and proof.

% \subsection{Theorem 1: Scaling Latent Skills}\label{app:theorem1}

\renewcommand{\thetheorem}{2}
\begin{theorem} [Scaling Latent Skills]
Consider the evaluation is composed of a number of \textit{independent} sub-tasks ("latent skills") where each of them observes a power law scaling law relationship with the amount of data contributing to this task ("equivalent data size"). Namely, 

\[
\mathcal{L} =\ell_0+\beta_a\cdot K_a^{-\gamma_a}+\beta_b\cdot K_b^{-\gamma_b}+\beta_c\cdot K_c^{-\gamma_c} + \cdots
\]

where scalar $K_j\geq 0$ denote equivalent data size for \textit{skill$_j$}, and constants $(\beta_j, \gamma_j)\geq 0$ are coefficients associated with \textit{skill$_j$}, respectively.
Define diagonal matrices for training data composition $\mathbf{N}=diag\{N_1,N_2, \cdots\}$ and skill data composition $\mathbf{K}=diag\{K_a,K_b, \cdots\}$. Consider training data from each domain $D_i$ contributes to these skills to varying degrees, given as $\mathbf{K} = \mathbf{A}\mathbf{N}$ where $\mathbf{A}_{ji}=c_{j,i}$ is the matrix for coefficients. 
Assume the amount of total training data from all domains is identical to the amount of total equivalent data for all skills, $\sum_j K_j = \sum_i N_i$. Assume there is a finite number of latent skills and data domains and $\mathbf{A}$ is a square matrix with full rank. 

For a training data scale $N=\sum_i N_i=\sum_j K_j$
, define optimal skill data composition $\mathbf{K^*}=diag\{K_a^*, K_b^*, \cdots\}$ as the minimizer of $\mathcal{L}$ s.t. $\sum_j K_j=N$ with corresponding optimal training data composition.
If we have optimal data compositions $\mathbf{N^*}(N^{(1)})=diag\{N_a^{(1)*}, N_b^{(1)*}, \cdots\}$ 
where its corresponding skill data composition $\mathbf{K^{(1)*}}=diag\{K_a^{(1)*}, K_b^{(1)*}, \cdots\}=\mathbf{A}\mathbf{N^*}(N^{(1)})$ minimizes $\mathcal{L}$ s.t. $\sum_j K_j=\sum_i N^{(1)*}=N^{(1)}$, and $\mathbf{N^*}(N^{(2)})=diag\{N_a^{(2)*}, N_b^{(2)*},...\}$ 
where its corresponding skill data composition $\mathbf{K^{(2)*}}=diag\{K_a^{(2)*}, K_b^{(2)*},...\}=\mathbf{A}\mathbf{N^*}(N^{(2)})$ minimizes $\mathcal{L}$ s.t. $\sum_j K_j^{(2)*}=\sum_i N^{(2)*}=N^{(2)}$ where $N^{(2)}\neq N^{(1)}$, 
then, other optimal data compositions $\mathbf{N^*}(N^{(3)})=diag\{N_a^{(3)*}, N_b^{(3)*},...\}$ 
where the corresponding skill data composition $\mathbf{K^{(3)*}}=diag\{K_a^{(3)*}, K_b^{(3)*}, \cdots\}=\mathbf{A}\mathbf{N^*}(N^{(3)})$ minimizes $\mathcal{L}$ s.t. $\sum_j K_j^{(3)*}=\sum_i N^{(3)*}=N^{(3)}$ where $N^{(3)}\neq N^{(2)}\neq N^{(1)}$ 
must satisfy
\[
\mathbf{N^*}(N^{(3)}) = \mathbf{N^*}(N^{(2)})[(\mathbf{N^*}(N^{(1)}))^{-1}\mathbf{N^*}(N^{(2)})]^k
\]
for some $k\in\mathbb{R}^+$.

 % For any third data composition $\mathbf{N}(N^{(3)})$, if there exists some constant $k\in\mathbb{R}^+$ such that
 %    \[
 %    \mathbf{N}(N^{(3)}) = \mathbf{N}^*(N^{(2)})[(\mathbf{N}^*(N^{(1)})) ^{-1} \mathbf{N}^*(N^{(2)})]^k, 
 %    \]
 %    then, $\mathbf{N}(N^{(3)})$ is the minimizer for data budget $N^{(3)}=\sum_{i=1}^m N^{(3)}_i$, given as 
 %         \[
 %        \mathbf{N}(N^{(3)}) = \arg\min_{\mathbf{N}} \left\{ \sum_{i=1}^m \beta_i N_i^{-\gamma_i} \Bigg| \sum_{i=1}^m N_i = N^{(3)} \right\}= \mathbf{N}^*(N^{(3)}).
 %    \]

\end{theorem}
\begin{proof}
By definition, we have
    \begin{equation*}
\begin{aligned}
\mathbf{A}\mathbf{N^*}(N^{(1)}) = \mathbf{K^{(1)*}},\quad
\mathbf{A}\mathbf{N^*}(N^{(2)}) = \mathbf{K^{(2)*}},\quad
\mathbf{A}\mathbf{N^*}(N^{(3)}) = \mathbf{K^{(3)*}}
\end{aligned}
\end{equation*}
From results of Theorem 1 in Section~\ref{thm:thm1}, we have
\begin{equation*}
    \mathbf{K^{(3)*}}=\mathbf{K^{(2)*}}[(\mathbf{K}^{(1)*})^{-1}\mathbf{K^{(2)*}}]^k
\end{equation*}
for some $k\in\mathbb{R}^+$, 
which gives
\begin{equation*}
    \mathbf{A}\mathbf{N^*}(N^{(3)})=(\mathbf{A}\mathbf{N^*}(N^{(2)}))[(\mathbf{A}\mathbf{N^*}(N^{(1)}))^{-1}\mathbf{A}\mathbf{N^*}(N^{(2)})]^k
\end{equation*}

Since $\mathbf{A}$ is invertible and $\mathbf{N}$ and $\mathbf{K}$ are diagonal matrices, naturally, 
\begin{equation*}
\begin{aligned}
(\mathbf{A}\mathbf{N^*}(N^{(1)}))^{-1} = (\mathbf{N^*}(N^{(1)}))^{-1}\mathbf{A}^{-1}
\end{aligned}
\end{equation*}
and we have
\begin{equation*}
    \mathbf{A}\mathbf{N^*}(N^{(3)})=\mathbf{A}\mathbf{N^*}(N^{(2)})[(\mathbf{N^*}(N^{(1)}))^{-1}\mathbf{A}^{-1}\mathbf{A}\mathbf{N^*}(N^{(2)})]^k=\mathbf{A}\mathbf{N^*}(N^{(2)})[(\mathbf{N^*}(N^{(1)}))^{-1}\mathbf{N^*}(N^{(2)})]^k
\end{equation*}

This directly gives
\begin{equation*}
    \mathbf{N^*}(N^{(3)})=\mathbf{A}^{-1}\mathbf{A}\mathbf{N^*}(N^{(2)})[(\mathbf{N}^*(N^{(1)}))^{-1}\mathbf{N^*}(N^{(2)})]^k=\mathbf{N^*}(N^{(2)})[(\mathbf{N}^*(N^{(1)}))^{-1}\mathbf{N^*}(N^{(2)})]^k
\end{equation*}
which completes the proof.

% Note that $\mathbf{N'}$ is a diagonal matrix, thus $\mathbf{A}\mathbf{N'}=\mathbf{N'}\mathbf{A}$.

% Then, combining the equations above, it gives
% \begin{equation*}
% \begin{aligned}
% \mathbf{N'} &= \mathbf{A}^{-1}\mathbf{K'}=\mathbf{A}^{-1}\mathbf{X}\mathbf{K}=(\mathbf{A}^{-1}\mathbf{K})\mathbf{X}=\mathbf{N}^{-1}\mathbf{X}
% \end{aligned}
% \end{equation*}
% \begin{equation*}
% \begin{aligned}
% \mathbf{N''} &= \mathbf{A}^{-1}\mathbf{X}\mathbf{K'}=(\mathbf{A}^{-1}\mathbf{K'})\mathbf{X}=(\mathbf{N}')^{-1}\mathbf{X}
% \end{aligned}
% \end{equation*}

% Straightforwardly, we end up having
% \begin{equation*}
% \begin{aligned}
% \mathbf{N}\mathbf{N'} &= \mathbf{N'}\mathbf{N''}
% \end{aligned}
% \end{equation*}
% which gives
% \begin{equation*}
% \begin{aligned}
% \mathbf{N''} &= (\mathbf{N'})^{-1}\mathbf{N}\mathbf{N'}
% \end{aligned}
% \end{equation*}

The above result does not require identifying the latent skills or observing skill data compositions $\mathbf{K}$. Rather, the theorem gives that as long as the coefficient matrix $\mathbf{A}$ is invertible, the scaling of $\mathbf{N}$ complies to the same scaling law as in Sec.~\ref{sec:thm}.
\end{proof}

\begin{remark} [what happens when $\mathbf{A}$ is not invertible.]
In general, if $\mathbf{A}$ is not invertible, scaling for optimal training data composition is not directly predictable. Specifically, if $\mathbf{A}$ does not have full rank, there exists redundant domains/data sources where their contribution to the skills are identical/exact multipliers of each other. Some data sources may not be needed at any scale; if $\mathbf{A}$ has more rows than columns (more domains than skills), this suggests multiple training data compositions can achieve the same skills data composition and the optimal training data compositions are non-unique (infinitely many).
If $\mathbf{A}$ has more columns than rows (more skills than domains), this means there are too many skills to optimize for. No optimal training data composition exists and one has to make trade-offs. If this is relevant to the practical needs, training data may be processed with additional techniques such as clustering and split into more different domains.
\end{remark}
\clearpage

\section{Experimental Details and Additional Results for Section~\ref{sec:eval}, Evaluation}\label{app:eval}

\subsection{Experimental Details on GPT-2 Large (774M)}
\label{sec:appendix_exp_details_gpt}

\textbf{Evaluation}
We test the perplexity on the held-out dataset, comprising 10K samples each from the 7 domains. For downstream tasks, we include: \texttt{BoolQ} \citep{clark2019boolq} (zero-shot), \texttt{HellaSwag} \citep{zellers2019hellaswag} (zero-shot, 10-shot), \texttt{PIQA} \citep{bisk2020piqa} (zero-shot), \texttt{TruthfulQA} \citep{lin2021truthfulqa} (zero-shot), \texttt{PubMedQA} \citep{jin2019pubmedqa} (10-shot), \texttt{CrowsPairs} \citep{nangia2020crows} (25-shot), and \texttt{ARC-Easy} \citep{clark2018think} (zero-shot). Additionally, \texttt{BBH Novel Concepts} \citep{srivastava2022beyond} task is added to the aggregated results for models trained beyond 10B tokens, making a total of 9 tasks. We select tasks that ensure the model's performance surpasses random guessing, spanning from question answering and commonsense inference to bias identification and scientific problem solving. These tasks provide a comprehensive assessment of model performance \citep{mehta2024openelm,gadre2024language}.
We adopt the evaluation framework from \citep{gao2021framework}. 

\textbf{Baselines}
We report results for our methods (\textsc{DDO} and \autoscale~) and 6 baselines–\textsc{Uniform}, \textsc{LLaMA weights} (curated), \textsc{DoReMi} (LLaMA weights initialization), \textsc{Data Mixing Laws from \citep{ye2024data}},  \textsc{DoReMi} from \citet{xie2024doremi} (uniform initialization), and \textsc{RegMix} from \citet{liu2024regmix}. Uniform weights uniformly sample data from all domains, resulting in the same number of training tokens from each domain. LLaMA weights are a set of curated domain weights heuristically tuned for training LLaMA-1/2 models. We implemented \textsc{DoReMi} proposed in \citep{xie2024doremi}. \textsc{DoReMi} trains two smaller-scale auxiliary models (proxy models). First, a reference model is trained with the dataset's original domain weights, which are the LLaMA weights for \texttt{RedPajama} dataset. Then, optimized domain weights are obtained by using a proxy model to minimize the worst-case excess loss across different domains. We train both auxiliary models for 50K steps. Implementation details are available in App.~\ref{sec:appendix_baseline_details}. Besides, we compare with 2 domain weights from existing literature, which are optimized on the same dataset, \texttt{RedPajama}, with similar Decoder-only LMs. \textsc{Data Mixing Laws} \citep{ye2024data} first performs a grid search on the space of possible data mixtures and records evaluation loss for proxy models trained on these mixtures. Then, the loss is interpolated with exponential functions to find the optimal domain weights for the proxy model. \textsc{DOGE} \citep{fan2023doge} also implements \textsc{DoReMi} \citep{xie2024doremi} with auxiliary models trained for 50K steps but with the reference model trained with uniform weights. \textsc{RegMix}~\citep{liu2024regmix} first trains an array of smaller, proxy models on different data mix and small data scales, abd fits a regression model between domain weights and evaluation loss. Then, the fitted regression model is used to predict the evaluation loss for all feasible domain weights to find the best-performing weights. We use the same pairs of domain weights and evaluation loss DDO used in optimizing domain weights for 774M Decoder-only LMs at 0.3B tokens to fit \textsc{RegMix}'s LightGBM regressor. The fitted LightGBM model is then used to optimize the evaluation loss over domain weights. We evaluate the model trained on these domain weights to present a complete landscape.

\paragraph{Model Training}

\texttt{GPT-2 Large} is a variant of the \texttt{GPT-2} architecture, featuring an embedding dimension of 1280, 36 transformer layers, and 20 attention heads. We rely on the Hugging Face Transformers library for implementation \citep{wolf2019huggingface}. Specific training hyperparameters are detailed in Table \ref{tab:gpt2-large hyperparams}.

\begin{table}[h]
\centering{
\begin{tabular}{lr}
\toprule
Architecture & gpt2\\
Optimizer & AdamW \\
 Tokenizer Vocabulary Size&$50257$\\

Batch Size Per Device & $1$\\
 Gradient Accumulation Steps&$10$\\
Maximum Learning Rate & 2e-4\\
LR Schedule & Linear \\
Weight Decay & 1e-2 \\
Warm-up Ratio& $10\%$\\
Epochs & $3$\\
GPU Hardware & 8x NVIDIA A100/8x NVIDIA H100\\
\bottomrule
\end{tabular}
\caption{The list of hyperparameters for \texttt{GPT-2 Large} pretraining.}

\label{tab:gpt2-large hyperparams}}
\end{table} 
% \call{table go to appendix}

% \kang{tokenizer vocabulary size}

\paragraph{Dataset Details}

The \texttt{RedPajama} dataset is available at: \url{https://huggingface.co/datasets/togethercomputer/RedPajama-Data-1T}.
The 7 domains involved are characterized as follows: 
% [list 7 domains, wiki, books, cc, c4 ...]

% \call{ list go to appendix}
\begin{itemize}
\item \textbf{\texttt{Commoncrawl}}: A vast repository of web-crawled data, providing a heterogeneous mix of internet text.
\item \textbf{\texttt{C4}}: The Colossal Clean Crawled Corpus, filtered to remove low-quality content, thus ensuring the reliability and cleanliness of the data.
\item \textbf{\texttt{GitHub}}: This domain includes a compilation of publicly available code repositories, offering a rich source of syntactic and semantic patterns inherent in programming languages.
\item \textbf{\texttt{Books}}: A collection of textual content from published books, providing diverse narrative styles and complex character developments.
\item \textbf{\texttt{ArXiv}}: Comprising scientific papers primarily from the fields of physics, mathematics, computer science, and quantitative biology, this domain offers high-quality, scholarly content.
\item \textbf{\texttt{Wikipedia}}: A well-organized and meticulously curated dataset of encyclopedia articles, delivering a broad spectrum of knowledge across multiple disciplines. We only use English samples with 'en' in meta-data.
\item \textbf{\texttt{StackExchange}}: This domain captures a variety of user-generated content from discussions and question-answer sessions across numerous technical topics.
\end{itemize}
% We apply the following preprocessing procedure to each domain-specific dataset to ensure consistency and quality for the pretraining of \texttt{GPT-2 Large}. \call{appendix}
Given copyright restrictions with the \texttt{Books} domain on Hugging Face, we have opted for an alternative source available at \url{https://yknzhu.wixsite.com/mbweb}.

For each domain, we ensure only samples with more than 1000 characters are retained. For each sample, the first 1000 characters are truncated, with the exception of the \texttt{ArXiv} and \texttt{GitHub} domains where we randomly extract a continuous block of 1000 characters. For the \texttt{Wikipedia} domain, we keep only those samples that are in English. Samples are selected without replacement, based on the computed data volume for each domain. Additionally, for each domain, a held-out dataset comprising 10K samples is reserved to evaluate the perplexity of the pretrained model. 
% \ys{Can go to manuscript if space allows}
% \call{appendix}

\subsection{Experimental Details on BERT (110M)}
\label{sec:appendix_exp_details_bert}

% 5 domains, arXiv, Amazon Reviews, Open Web Text Corpus (OWTC), Wikipedia, books, 
% excluded programming language such as GitHub and StackExchange
% Amazon Reviews for sentiment analysis
% bert-base-uncased training from scratch
We evaluate the model's MLM loss on held-out validation datasets, comprising 10K samples each from the 5 training domains. Additionally, as an auxiliary evaluation, we test the MLM loss on 3 non-training held-out domains. To be consistent with the perplexity loss used in CLM, we report the exponential cross-entropy loss for MLM. We evaluate the model's task performance on \texttt{GLUE} benchmark \citep{wang2018glue} (with 8 diverse tasks for natural language understanding (NLU)) and \texttt{SQuAD} \citep{rajpurkar2016squad} (a large-scale QA dataset). Uniform weights are used as the baseline.

\paragraph{Model Training} We employ the \texttt{BERT-base-uncased} model from the Hugging Face Transformers library. Originally, \texttt{BERT}'s pretraining scheme involved MLM and next sentence prediction (NSP); however, in our experiments, we exclusively utilize MLM. Detailed training hyperparameters can be found in Table \ref{tab:bert hyperparams}.
\begin{table}[h]
\centering{
\begin{tabular}{lr}
\toprule
Architecture & bert-base-uncased \\
Max Token Length & $300$ \\
Mask Token Percentage & $15$\% \\
Optimizer & AdamW \\

Batch Size Per Device & $12$ \\
Devices & $4$ \\
Maximum Learning Rate & 1e-4 \\
LR Schedule & Linear \\
Weight Decay & 1e-2 \\
Warm-up Steps & $3000$\\
Epochs & $1\sim4$ \\
GPU Hardware & 4x NVIDIA RTX A6000\\
\bottomrule
\end{tabular}

 \caption{The list of hyperparameters for \texttt{BERT} pretraining.}
\label{tab:bert hyperparams}}
\end{table}

\paragraph{Dataset Details}

The 5 domains of training data utilized are listed as follows:

\begin{itemize}
\item \textbf{\texttt{Amazon Reviews}}: A compilation of customer reviews from Amazon, widely utilized in sentiment analysis studies, available at: \url{https://huggingface.co/datasets/amazon_us_reviews}.
\item \textbf{\texttt{Arxiv}}: Comprises 1.7 million articles from arXiv, available at: \url{https://www.tensorflow.org/datasets/catalog/scientific_papers}.
\item \textbf{\texttt{Books}}: A corpus of 11,038 novels by unpublished authors across 16 genres, available at: \url{https://yknzhu.wixsite.com/mbweb}.
\item \textbf{\texttt{Wikipedia}}: Offers datasets extracted from Wikipedia in various languages, available at: \url{https://www.tensorflow.org/datasets/catalog/wikipedia}. We only use English samples with 'en' in meta-data.
\item \textbf{\texttt{Open WebText Corpus (OWTC)}}: A corpus of English web texts from Reddit posts, available at: \url{https://skylion007.github.io/OpenWebTextCorpus/}.
\end{itemize}

3 held-out non-training domains used in the evaluation include:
\begin{itemize}
    \item \textbf{\texttt{Pubmed}}: Features 19,717 diabetes-related publications from the PubMed database, organized into three classes and linked by a network of 44,338 citations, available at: \url{https://www.tensorflow.org/datasets/catalog/scientific_papers}
    \item \textbf{\texttt{News}}: Comprises a significant collection of news articles derived from \texttt{CommonCrawl}, specifically from 5000 news domains indexed by Google News, available at: \url{https://github.com/rowanz/grover/blob/master/realnews/README.md}
    \item \textbf{\texttt{GitHub}}: A curated selection from the \texttt{RedPajama} dataset, this segment includes an array of open-source code projects, available at: \url{https://huggingface.co/datasets/togethercomputer/RedPajama-Data-1T}
\end{itemize}

\subsection{Implementation Details for Baselines}
\label{sec:appendix_baseline_details}

\paragraph{Implementation details} We followed the official implementation\footnote{\url{https://github.com/sangmichaelxie/doremi}} of \textsc{DoReMi} for our experiments. We evaluated two sets of reference domain weights: (1) the domain weights utilized in the LLaMA-2 paper~\cite{touvron2023llama} (referred to as LLaMA weights), and (2) uniform weights. Both the reference and proxy models have 120M parameters and are trained from scratch. We use \texttt{GPT-2} tokenizer with a vocabulary size of roughly 50K. For LLaMA weights, we train each model for 20K, 50K and 200K steps for comparison. For uniform weights, we train each model for 10K, 20K and 50K steps. Refer to Table \ref{tab:doremi} for detailed hyperparameters.  The effect of reference weights on the output \textsc{DoReMi} is discussed in Fig.\ref{fig:doremi_ref_weights}.
% \call{appendix}

\begin{table}[h]
\centering{
\begin{tabular}{lr}
\toprule
Architecture & Decoder-only LM \\
Max Token Length & $1024$ \\
Optimizer & AdamW \\

Batch Size Per Device & $8$ \\
Devices & $8$ \\
Maximum Learning Rate & 2e-4 \\
LR Schedule & Linear \\
Weight Decay & 1e-2 \\
Warm-up Steps & $3000$\\
Epochs & $1$ \\
GPU Hardware & 8x NVIDIA RTX A6000\\
\bottomrule
\end{tabular}
\caption{The list of hyperparameters for \textsc{DoReMi}.}

\label{tab:doremi}}
\end{table}
% \call{appendix}

% DoReMi

% How it works, 

% implementation details (gituhub repo, main library versions, training hyperparameters, hardware, runtime, data, steps, compute)

% data pre-processing, reference model, proxy model, steps

% thoughts and comments

% , LLaMA weights, (+ others)

% [For details and additional results see App.~C2.1 and C2.2 ]

\subsection{Evaluation Details}
\label{sec:appendix_evaluation}

\paragraph{GPT/CLM}
The following tasks are considered for downstream performance evaluation, in line with the setup from \citep{mehta2024openelm,gadre2024language}. For few-shot tasks, the demonstrations are sampled at random.
\begin{itemize}
    \item \textbf{\texttt{BoolQ}} \citep{clark2019boolq} consists of a question-answering format that requires binary yes/no answers.
    \item \textbf{\texttt{HellaSwag}} \citep{zellers2019hellaswag} challenges models on their ability to make commonsense inferences.
    \item \textbf{\texttt{PIQA}} \citep{bisk2020piqa} focuses on evaluating a model's commonsense reasoning regarding physical interactions. 
    \item \textbf{\texttt{TruthfulQA}} \citep{lin2021truthfulqa} is designed to assess the ability of models to generate truthful and factual responses.
    \item \textbf{\texttt{PubMedQA}} \citep{jin2019pubmedqa} offers a dataset for evaluating question-answering in the biomedical domain.
    \item \textbf{\texttt{CrowsPairs-English}} \citep{nangia2020crows} tests models on their ability to identify and correct stereotypical biases in English text.
    \item \textbf{\texttt{ARC-Easy}} \citep{clark2018think} presents a set of relatively simpler scientific reasoning questions, aimed at evaluating a model's basic understanding of scientific principles.
    \item \textbf{\texttt{BigBench-Novel Concepts}} \citep{srivastava2022beyond} serves as a test of the model's creative abstraction skills, challenging it to make sense of scenarios that it could not have memorized during training.

\end{itemize}

% \call{@Yifan: help here.} 

\paragraph{BERT/MLM}

For each task, we conduct supervised fine-tuning on the corresponding training data and test the fine-tuned model on the validation data. The hyperparameters for supervised fine-tuning are given in Table \ref{tab: bert evaluation}.

\begin{table}[h]
\centering{
\begin{tabular}{lr}
\toprule
Architecture & bert-base-uncased \\
Max Token Length & $128$ \\
Batch Size Per Device & $8$ or $300$ \\
Optimizer & AdamW \\
Devices & $4$ \\
Maximum Learning Rate & 2e-5 or 5e-5 \\
Epochs & $3$ \\
GPU Hardware & 4x NVIDIA RTX A6000\\
\bottomrule
\end{tabular}
}
\caption{The list of hyperparameters for supervised fine-tuning of \texttt{BERT}.}
\label{tab: bert evaluation}

\end{table}

\subsection{Additional Results on GPT-2 Large (774M)}
\label{sec:appendix_additional_results_gpt}

\begin{figure}[h!] 
    \centering
    \begin{subfigure}[b]{0.42\textwidth}
        \includegraphics[width=\textwidth]{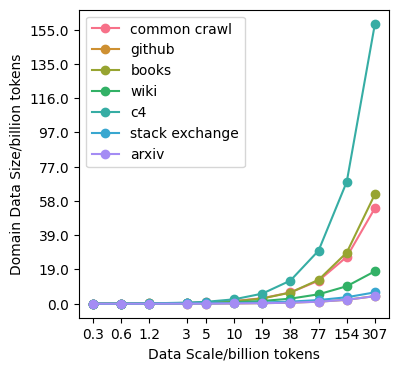}
        \caption{\autoscale~-predicted optimal data quantity for each domain as training data scales up.}
    \end{subfigure}
    % \hfill
    \hspace{1em}
    \begin{subfigure}[b]{0.51\textwidth}
        \includegraphics[width=\textwidth]{gptfigs/gptdmw.png}
        \caption{\autoscale~-predicted optimal domain weights as training data scales up.}
    \end{subfigure}
    \caption{\autoscale~-predicted domain weights for training 774M Decoder-only LMs. Optimal data quantity for each domain grows in exponential-style functions with training data scale (left) where data sources with diverse samples (e.g., \texttt{C4}) are upweighted relative to domains with standard format (e.g., \texttt{Wikipedia}).} 
    \label{fig:figure4}
\end{figure}

Fig.~\ref{fig:figure4} depicts \autoscale-predicted domain weights for training 774M Decoder-only LMs (GPT-2 Large). Optimal data quantity for each domain grows in exponential-style functions with training data scale (left) where data sources with diverse samples (e.g., \texttt{C4}) are upweighted relative to domains with standard format (e.g., \texttt{Wikipedia}).

Fig.~\ref{fig:gpt2_additional_3} shows that when training on up to 5B tokens, \autoscale~-predicted weights decreases val loss at least $25\%$ faster than any baseline with up to $37\%$ speed up. 

Fig.~\ref{fig:weights_gpt} visualizes domain weights used for training \texttt{GPT-2 Large}, given by different methods.

% Table~\ref{tab:gpt2_additional_1} examines the domain-specific perplexity of \texttt{GPT-2 Large} trained on 3 billion tokens, respectively. Notably, \autoscale~ achieves the lowest average validation perplexity and significantly reduces the perplexity in the worst-performing domains.

Fig.~\ref{fig:doremi_ref_weights} visualizes \textsc{DoReMi} optimized domain weights with different reference weights and training steps. Training proxy/reference models for different steps gives different weights. It is unclear which weights are optimal. \textsc{DoReMi} recommends 200k steps, which equals >100B tokens in the default setup. Since optimization was conducted relative to the reference weights, reference weights have a profound impact on \textsc{DoReMi}'s output.

%######

\begin{figure}[h!]
    \centering
    \begin{subfigure}[b]{0.48\textwidth}
        \includegraphics[width=\textwidth]{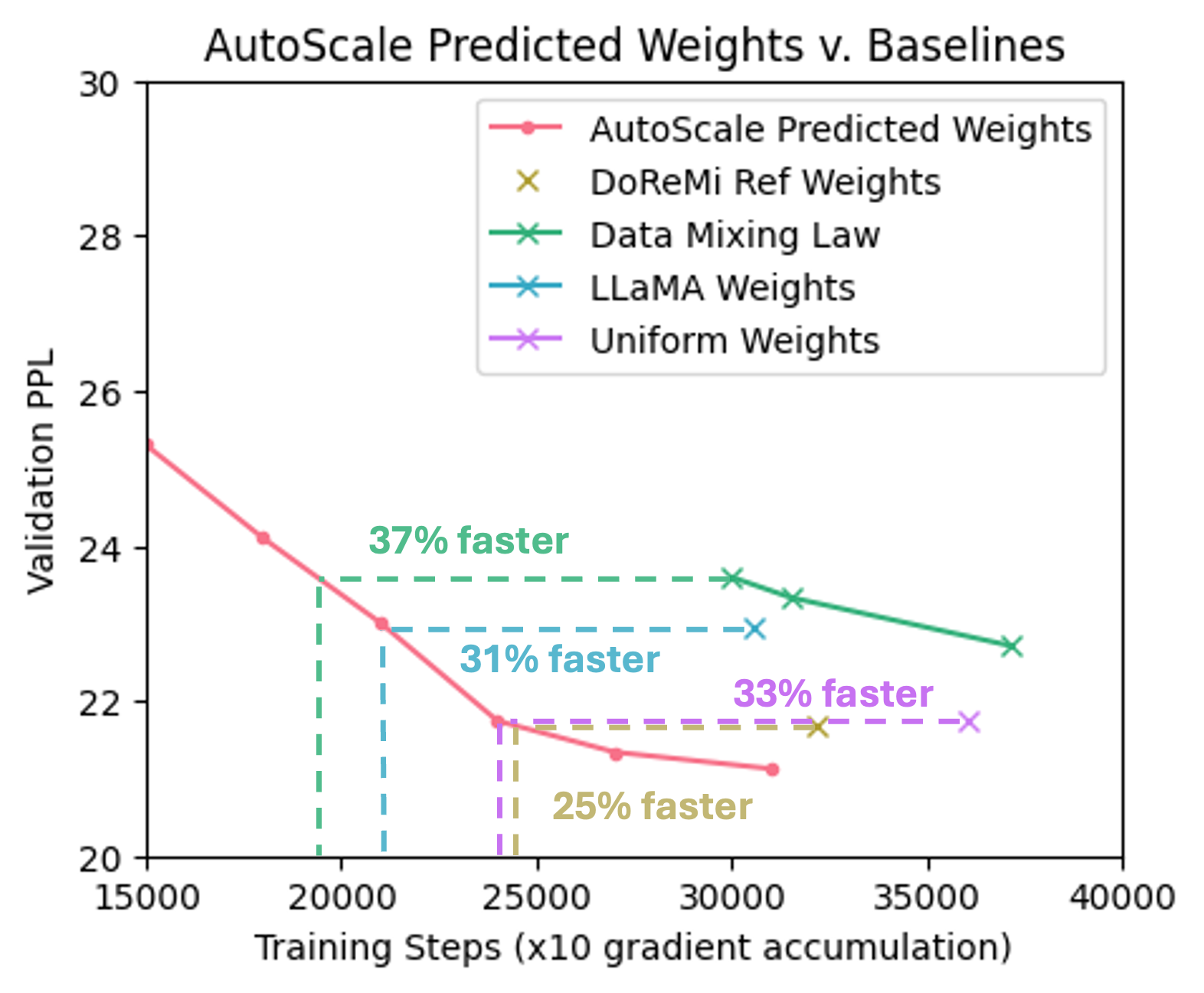}
        \caption{Training Decoder-only LMs for 3B tokens.}
        \label{fig:figure7a}
    \end{subfigure}
    % \hfill
    % \hspace{1em}
    \begin{subfigure}[b]{0.48\textwidth}
        \includegraphics[width=\textwidth]{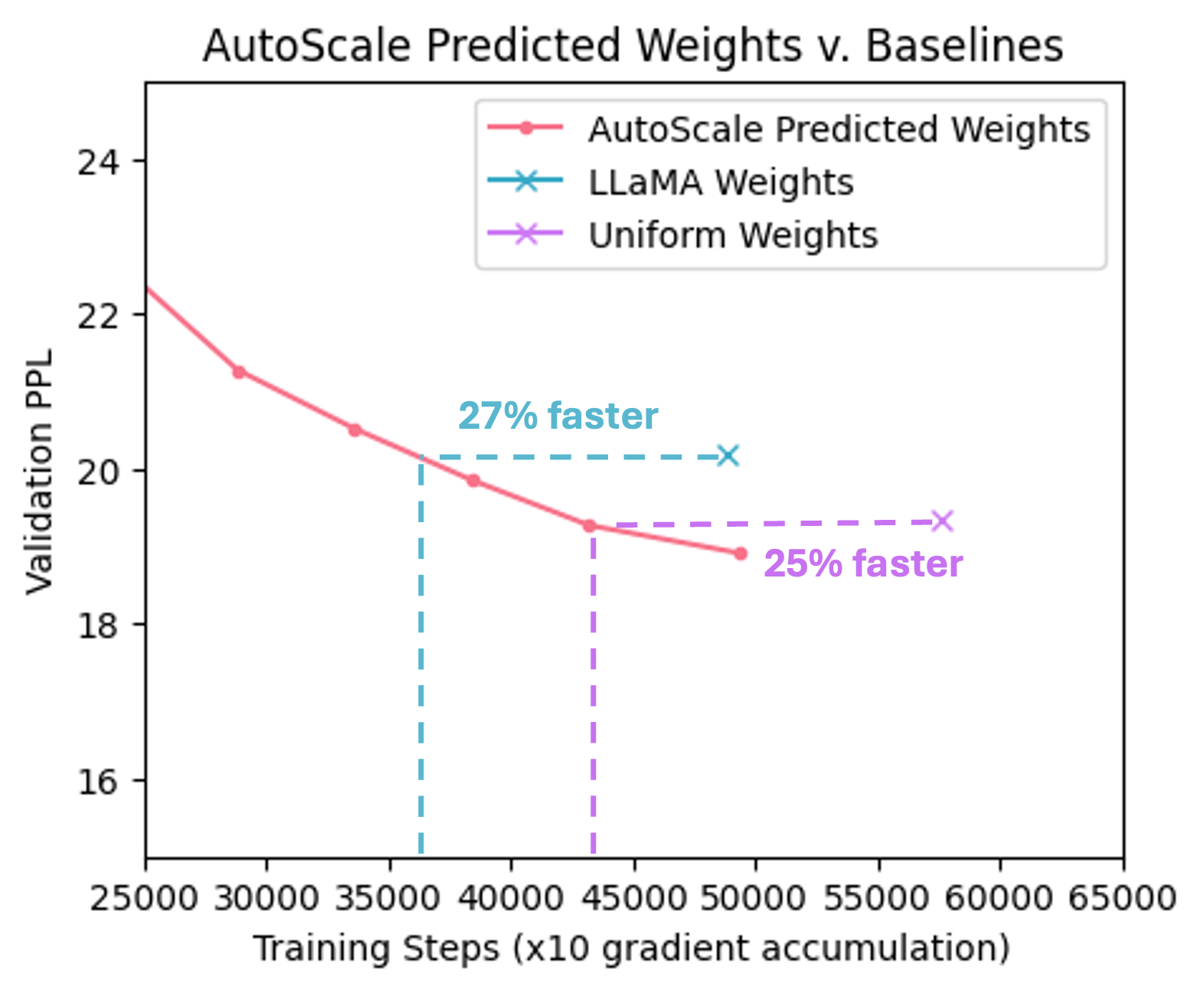}
        \caption{Training Decoder-only LMs for 5B tokens.}
        \label{fig:figure7b}
    \end{subfigure}
    \caption{\autoscale~-predicted weights decreases val loss at least $25\%$ faster than any baseline with up to $37\%$ speed up. Despite LLaMa weights being very different from uniform weights, they yield highly similar training efficiency at these data scales.}
    \label{fig:gpt2_additional_3}
\end{figure}

\begin{figure}[t!]
\begin{center}
  \includegraphics[width=0.5\textwidth]{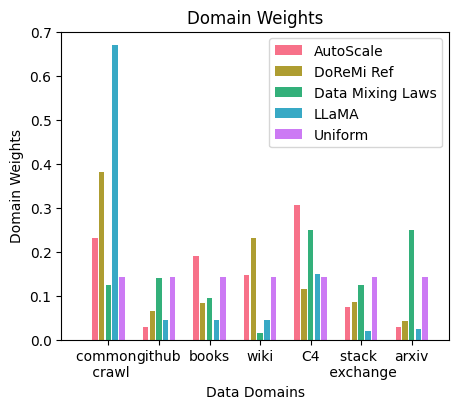}
  \vspace{-1em}
  \caption{Domain Weights used for training 774M Decoder-only LMs for 3B tokens. (Domain weights for \textsc{Data Mixing Laws} and \textsc{DoReMi} are from references \citep{ye2024data} and \citep{fan2023doge}, respectively, which are implemented on the same datasets/data domains with highly similar model architecture/model size/tokenizers.)
  }\label{fig:weights_gpt}
  \vspace{-1em}
  \end{center}
\end{figure}% \vspace{-1em}

\begin{figure}[h!]
    \centering
    \begin{subfigure}[b]{0.45\textwidth}
        \includegraphics[width=\textwidth]{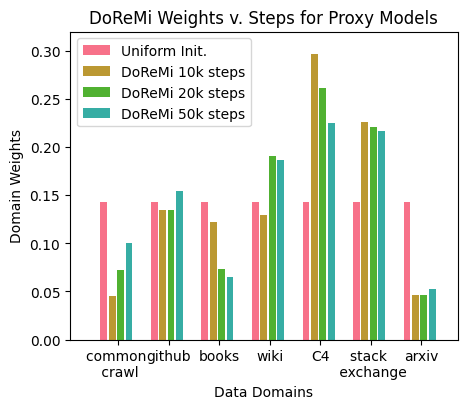}
        \caption{with Uniform Reference Weights}
    \end{subfigure}
    % \hfill
    \hspace{1em}
    \begin{subfigure}[b]{0.45\textwidth}
        \includegraphics[width=\textwidth]{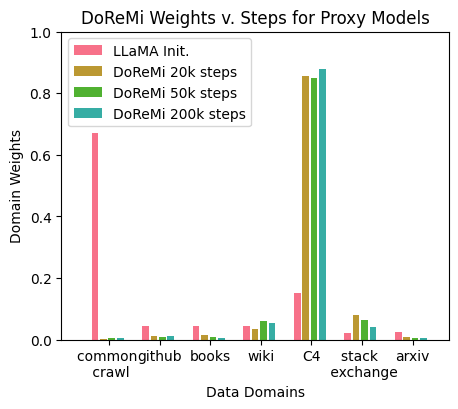}
        \caption{with LLaMA Reference Weights (Default)}
    \end{subfigure}
    \caption{\textsc{DoReMi} with different reference weights and steps.  Training proxy/reference models for different steps gives different weights. It is unclear which weights are optimal. \textsc{DoReMi} recommends 200k steps, which equals >100B tokens in the default setup. Since optimization was conducted relative to the reference weights, reference weights have a profound impact on \textsc{DoReMi}'s output.}
    \label{fig:doremi_ref_weights}
\end{figure}

\subsection{Additional Results on BERT (110M)}
\label{sec:appendix_additional_results_bert}

Fig.~\ref{fig:figure30}(b) shows the results on fitting validation loss with power-law functions, directly approximating how loss changes with each domain's data quantity. Compared to \texttt{BERT} models trained with MLM (right), \texttt{GPT} models trained with CLM (left) demonstrate a much stronger response to domain reweighting. In final results, \texttt{GPT}/CLM achieved $> 2\times$ speed-up margins relative to uniform weights compared to \texttt{BERT}/MLM. 

Fig.~\ref{fig:bert_additional} depicts the \autoscale~-predicted domain weights for training \texttt{BERT}. It is evident that optimal data quantity for each domain grows in exponential-style functions with training data scale where data sources with diverse samples (e.g., \texttt{WebText}) are upweighted relative to domains with standard format (e.g., \texttt{ArXiv}).

Table~\ref{table10} shows \autoscale~ notably improving training efficiency for \texttt{BERT} models on all scales–even for a considerably large scale, 288k steps, the speedup margin remains visible.

\begin{figure}[h!]
    \centering 
    \begin{subfigure}[b]{0.45\textwidth}
        \includegraphics[width=\textwidth]{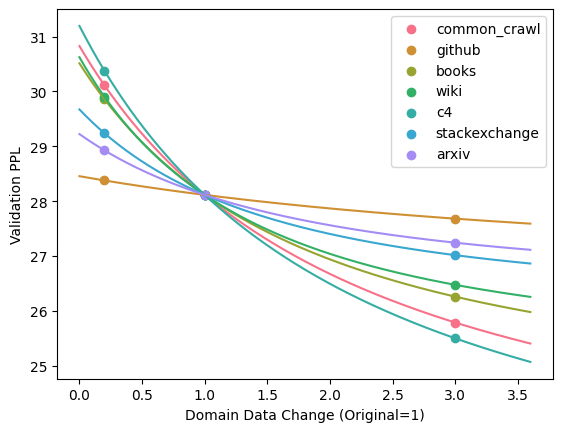}
        \caption{774M Decoder-only LMs (\texttt{GPT-2 Large})}
    \end{subfigure}
    % \hfill
    \hspace{2em}
    \begin{subfigure}[b]{0.45\textwidth}
        \includegraphics[width=\textwidth]{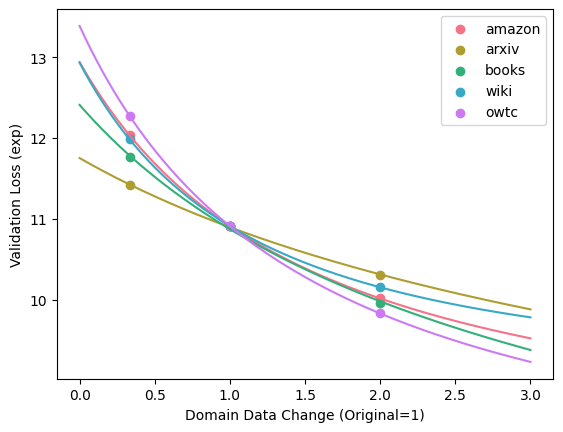}
        \caption{Encoder-only LMs (\texttt{BERT-case})}
    \end{subfigure}
    \caption{Fitting validation loss with power-law functions, directly approximating how loss changes with each domain's data quantity. Compared to \texttt{BERT} models trained with MLM (right), \texttt{GPT} models trained with CLM (left) demonstrate a much stronger response to domain reweighting. In final results, \texttt{GPT}/CLM achieved $> 2\times$ speed-up margins relative to uniform weights compared to \texttt{BERT}/MLM. }
    \label{fig:figure30}\vspace{-1em}
\end{figure}

\begin{figure}[h!]
    \centering
    \begin{subfigure}[b]{0.4\textwidth}
        \includegraphics[width=\textwidth]{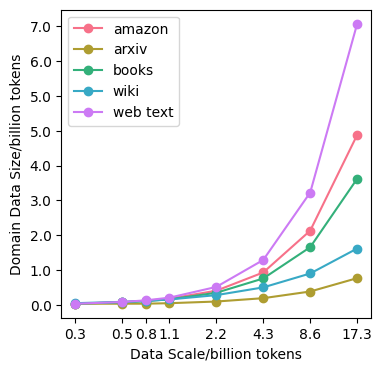}
        \caption{\autoscale~-predicted optimal data quantity for each domain as training data scales up.}
    \end{subfigure}
    % \hfill
    \hspace{1em}
    \begin{subfigure}[b]{0.51\textwidth}
        \includegraphics[width=\textwidth]{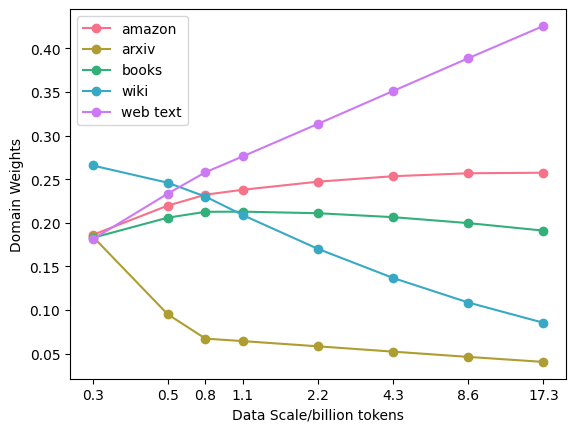}
        \caption{\autoscale~-predicted optimal domain weights as training data scales up.}
        \label{fig:bert_additional_sub_2}
    \end{subfigure}
    \caption{\autoscale~-predicted domain weights for training Encoder-only LMs (\texttt{BERT}). Optimal data quantity for each domain grows in exponential-style functions with training data scale (left) where data sources with diverse samples (e.g., \texttt{WebText}) are upweighted relative to domains with standard format (e.g., \texttt{ArXiv}).}
    \label{fig:bert_additional}
\end{figure}

\begin{table}[h!]\vspace{-0em}\centering{\resizebox{0.7\linewidth}{!}{
\begin{tabular}{lccccc}
\toprule
\textbf{Data Scale/steps} & 18k       & 36k       & 72k        & 144k       & 288k       \\\midrule
Final Loss (exp)  & 38.32     & 16.94     & 10.97      & 8.13       & 6.30       \\
Steps Saved       & 5k (28\%) & 5k (14\%) & 10k (14\%) & 20k (14\%) & 20k (10\%)\\
\bottomrule
\end{tabular}}}\caption{\autoscale~ notably improving training efficiency for \texttt{BERT} models on all scales–even for a considerably large scale, 288k steps, the speedup margin remains visible.}\label{table10} 
\end{table}

\section{Extended Discussions}\label{app:diss}

\paragraph{Intuition for Stage 2 of \autoscale, Optimal mix projection.}

Consider a stylized case where each training domain is independent. In this scenario, the validation performance attributable to each domain scales with the amount of training data from that same domain. Domains exhibit \textbf{different scaling behaviors:}

\begin{itemize}
\item As the training budget increases, the loss associated with some domains (e.g., 'easier' domains like standard texts from Wikipedia) decreases rapidly at first and then plateaus. If the total training data budget is small, it is sensible to allocate more budget to such domains, as they offer a greater initial reduction in validation loss.

\item Conversely, for other domains (e.g., diverse sources like Common Crawl), the loss decreases more slowly but steadily with an increasing training budget, continuing to provide benefits even when the loss from other domains has plateaued.

\item With larger compute budgets, it becomes more beneficial to allocate resources to these steadily improving domains rather than adding more data to 'easy' domains that are already saturated. (As Figure \ref{fig:evidence} illustrates, optimal compositions are scale-dependent.) The key insight, which we formalize in a theorem, is that under certain independence assumptions, the evolution of optimal domain compositions with increasing training data budgets can be analytically derived from their individual scaling laws. This allows us to predict optimal compositions at larger data budgets.
\end{itemize}

We then relax the assumption of strict domain independence. We theoretically prove that if the validation data requires a number of independent 'latent skills,' and data from each training domain contributes to one or more of these skills, our previous theory still holds in the same form. Thus, direct domain independence is not a strict requirement. This provides a complete theoretical grounding for our \autoscale tool, which captures the evolution of optimal domain weights with the training data budget and offers predictability for larger scales. Finally, our empirical evaluations demonstrate that the weights predicted by AutoScale offer clear advantages over domain weights optimized without considering this scale-dependence.the discussions and future work.

\paragraph{Number of predictors (scaling law components)}

We conducted ablation studies with DML on optimizing domain weights under a fixed compute dudget, following the same experiment procedure in Section \ref{sec:ddo} to pre-train GPT-2 Large models on the RedPajama dataset. According to DML’s procedure, we fit a separate exponential function between domain weights and validation perplexity for each of the 7 domains, resulting in 7 separate predictors. When predicting for the average validation loss, we use each predictor to predict the validation perplexity for the respective domain and take the average of these predictions as the final output. When fitting on proxy training on random data mixtures, this prediction approach yields an average absolute relative error (AAR) = 4.46\%. In comparison, similar to the treatment in the proposed DDO, we also fitted a single function between domain weights and the average validation perplexity using DML’s exponential formula. We experimented with using this function to directly predict the average validation perplexity. When fitting on proxy training on random data mixtures, this prediction approach yields an AAR = 1.61\%.

When using the fitted predictor to derive optimal data mixtures, we realized DML's formula does not allow optimize domain weights over a single predictor. For a validation domain, the predicted validation loss 
 is given as the exponentiation of weighted average of domain weights plus some constant (Eq. (1) from \citet{ye2024data}). Since the domain weights are combined linearly before the exponentiation, if there is a single validation domain, minimizing the predicted loss will lead to a degenerate solution where all training data budgets is assigned to a single domain. Thus, it is necessary for DML to incorporate a number of separate predictors to introduce the crucial nonlinearity needed to model the interactions between different domains. The power-law formula in DDO is able to capture the nonlinearity in domain interactions in a single predictor, leading to lighter parametrization and enhanced efficiency

% \subsection{Runtime Analysis}\label{runtime}
% Training a \texttt{GPT-2 large} model from scratch for 3B tokens requires 15.5 hours on 8x NVIDIA A100 40GB SXM GPUs or 9 hours on 8x NVIDIA H100 80GB GPUs. Training time increases linearly with the number of training tokens on both types of GPUs.

% Training \texttt{BERT-base} models takes 2 hours for every 18k steps on 4x NVIDIA A6000 48GB GPUs. Computational time grows linearly with the number of training steps.

% Training reference models for \textsc{DoReMi} takes one hour for every 10K steps on 8x NVIDIA A6000 48GB GPUs. Computational time grows linearly with the number of training steps. Similar runtime for training proxy models for \textsc{DoReMi}. 

% % \section{Discussions}

% % Retraining-based gradient optimization is the most effective. Gradients must be calculated directly between the final objective (e.g., average exponential loss) and the optimization variable (data composition). Assuming data from each domain only/mostly affects its own performance leads to ineffective data selection.
% % First-order gradient optimization needs to use line search to determine the optimal update at each step. The second-order gradient directly provides the optimal stepsize for each update.

% % % \subsection{Connection to Curriculum Learning}

% % \subsection{Reproducible \kang{rewrite}}

% The effect of noise must be carefully taken care of. Signal-to-noise ratio is a helpful tool for sanity checks. Some baseline papers produce false negative results for small signals overwhelmed by high noise.
\end{appendices}
\end{document}